\documentclass[sigconf,natbib=true]{acmart}
\usepackage[utf8x]{inputenc}
\usepackage{microtype}
\usepackage{thm-restate}
\usepackage{paralist}
\usepackage{url}
\usepackage{array}
\usepackage{multirow}
\usepackage{amsfonts,bm}
\usepackage{amssymb}
\usepackage{amsthm}
\usepackage{todonotes}
\usepackage[export]{adjustbox}
\usepackage[small,compact]{titlesec}
\usepackage{ifthen}
\usepackage{todonotes}
\usepackage{pgfplots}
\usepackage{wrapfig}
\usepackage{pgf}	
\usepackage{tikz}
\usetikzlibrary{automata,arrows,intersections,positioning,calc,shapes}
\tikzstyle{vertex}=[minimum size=10pt]
\tikzstyle{class green}=[fill=white, draw=black, inner sep=0pt, minimum size=10pt, regular polygon, regular polygon sides=3]
\tikzstyle{class blue}=[fill=white, draw=black, inner sep=0pt, minimum size=10pt, regular polygon, regular polygon sides=4]
\tikzstyle{class red}=[fill=white, draw=black, inner sep=0pt, minimum size=10pt, regular polygon, regular polygon sides=5]
\tikzstyle{sensor}=[fill=white, draw=black, inner sep=0pt, minimum size=10pt, regular polygon, regular polygon sides=6]
\usepackage{siunitx}
\usepackage{wrapfig}
\usepackage{commath}
\usepackage{comment}
\usepackage{nicefrac}

\usepackage{graphicx}
\usepackage{pgfplots}
\usepackage{caption}
\usepackage{subcaption}
\usepackage{algorithm,algorithmicx}
\usepackage[noend]{algpseudocode}

\usepackage{xcolor}
\usepackage{enumitem}

\usepackage{rotating}

\usepackage{csvsimple}
\usepackage{longtable}

\usepackage{ifthen}
\newboolean{nonstationary}
\setboolean{nonstationary}{false}
\newboolean{complexity}
\setboolean{complexity}{false}
\newcounter{myexample}

\newcommand{\bp}[1]{\ensuremath{{#1}_{\updownarrow}}}
\newcommand{\R}{\mathbb{R}}
\newcommand{\N}{\mathbb{N}}

\renewcommand{\vec}[1]{{\bf #1}}
\renewcommand{\phi}{\varphi}
\newcommand{\mat}[1]{{\bf #1}}

\newcommand{\bigO}[1]{\ensuremath{\mathcal{O}\del{#1}}}

\newcommand{\Ppure}{\mathcal{P}_\mathrm{pure}}
\newcommand{\Ppareto}{\mathcal{P}_\mathrm{Pareto}}

\newcommand{\Popt}{\mathcal{P}_\mathrm{opt}}

\newcommand{\Vpure}{\mathcal{V}_\mathrm{pure}}
\newcommand{\Vpareto}{\mathcal{V}_\mathrm{Pareto}}

\newcommand{\Vopt}{\mathcal{V}_\mathrm{opt}}

\newcommand{\grad}{\nabla}
\newcommand{\mycomment}[1]{}

\DeclareMathOperator*{\Ex}{Ex}

\usepackage{todonotes}
\newcommand{\idvt}{\,{\rm I\kern-0.55em 1 }}
\usepackage{thm-restate}
\newboolean{useComments}
\setboolean{useComments}{false}
\definecolor{dark_purple}{rgb}{0.1, 0.0, 0.4}
\definecolor{dark_green}{rgb}{0.0,0.2,0.5}
\definecolor{dark_red}{rgb}{0.85,0, 0}
\ifthenelse{\boolean{useComments}}{%
	\setlength{\marginparwidth}{10cm}
	\newcommand{\ds}[1]{\todo[inline,color=white!40,bordercolor=white]{\textcolor{teal}{\textbf{Dimitri:}\textmd{\;#1}}}} 
	\newcommand{\hh}[1]{\todo[inline,color=white!40,bordercolor=white]{\textcolor{red}{\textbf{Holger:}\textmd{\;#1}}}}
	\newcommand{\vh}[1]{\todo[inline,color=white!40,bordercolor=white]{\textcolor{violet}{\textbf{Vahid:}\textmd{\;#1}}}}
	\newcommand{\pb}[1]{\todo[inline,color=white!40,bordercolor=white]{\textcolor{purple}{\textbf{Peter:}\textmd{\;#1}}}}
	\newcommand{\td}[1]{\todo[inline,color=dark_red!10,bordercolor=white]{\textcolor{dark_red}{\textbf{ToDo:}\textmd{\;#1}}}} 
}{%
\newcommand{\ds}[1]{}
\newcommand{\hh}[1]{}
\newcommand{\vh}[1]{}
\newcommand{\pb}[1]{}
\newcommand{\td}[1]{}
}

\newboolean{long}   

\newboolean{author} 
\newboolean{pdfinfo}  
\newboolean{acknowledgments}  
\setboolean{pdfinfo}{true} 
\setboolean{acknowledgments}{true} 
 
\makeatletter
\def\thm@space@setup{\thm@preskip=3pt
	\thm@postskip=3pt}
\makeatother

\ifthenelse{\boolean{pdfinfo}}{\pdfinfo{
/Title (Multi-Objective Approaches to Markov Decision Processes with
Uncertain Transition Parameters)
/Author (Dimitri Scheftelowitsch, Peter Buchholz, Vahid Hashemi, Holger Hermmans)}}{}
\title[Multi-Objective Approaches for Uncertain MDPs]{Multi-Objective Approaches to Markov Decision Processes with
	Uncertain Transition Parameters}
\author{Dimitri Scheftelowitsch}
\affiliation{
  \institution{Informatik IV, TU Dortmund}
  \city{Dortmund}
  \country{Germany}
}
\email{dimitri.scheftelowitsch@cs.tu-dortmund.de}
\author{Peter Buchholz}
\affiliation{
  \institution{Informatik IV, TU Dortmund}
  \city{Dortmund}
  \country{Germany}
}
\email{peter.buchholz@cs.tu-dortmund.de}
\author{Vahid Hashemi}
\affiliation{
  \institution{Saarland University -- Computer Science}
  \streetaddress{Saarland Informatics Campus}
  \city{Saarbr\"ucken}
  \country{Germany}
}
\email{hashemi@depend.uni-saarland.de}
\author{Holger Hermanns}
\affiliation{
  \institution{Saarland University -- Computer Science}
  \streetaddress{Saarland Informatics Campus}
  \city{Saarbr\"ucken}
  \country{Germany}
}
\email{hermanns@depend.uni-saarland.de}

\newboolean{final}
\setboolean{final}{false}

\ifthenelse{\boolean{final}}{%
  \copyrightyear{2017}%
  \acmYear{2017}%
  \setcopyright{acmcopyright}%
  \acmConference[VALUETOOLS 2017]{11th EAI International Conference on Performance
  Evaluation Methodologies and Tools}{December 5--7, 2017}{Venice, Italy}%
  \acmBooktitle{VALUETOOLS 2017: 11th EAI International Conference on Performance
  Evaluation Methodologies and Tools, December 5--7, 2017, Venice, Italy}%
  \acmPrice{15.00}%
  \acmDOI{10.1145/3150928.3150945}%
  \acmISBN{978-1-4503-6346-4/17/12}%
}{
  \acmPrice{}
  \acmConference[]{}{}{}
  \acmBooktitle{}
  \acmDOI{}
  \acmISBN{}
}


\begin{document}

\begin{abstract}
	Markov decision processes (MDPs) are a popular model for
        performance analysis and optimization of stochastic systems. The
        parameters of stochastic behavior of MDPs are estimates 
        from empirical observations of a system; their values are
        not known precisely. Different types of MDPs with uncertain, imprecise
        or bounded transition rates or probabilities and rewards exist in the
        literature.

        Commonly, analysis of models with uncertainties amounts to searching for
        the most
	robust policy which means that the goal is to generate a policy with the
        greatest lower bound on performance (or, symmetrically, the lowest upper
        bound on costs). However, hedging against an unlikely worst case
        may lead to losses in other situations. In general, one is
	interested in policies that {\em behave well\/} in \emph{all} situations 
	which results in a multi-objective view on decision making.

	In this paper, we consider policies for the expected discounted reward 
	measure of MDPs with uncertain parameters. In particular, the approach 
	is defined for bounded-parameter MDPs
	(BMDPs)~\cite{DBLP:journals/ai/GivanLD00}. In this setting the worst, 
	best and average case performances of a policy are analyzed 
	simultaneously, which yields a multi-scenario multi-objective 
	optimization problem. The paper presents and evaluates approaches to 
	compute the {\em pure\/} Pareto optimal policies in the value vector 
	space.
\end{abstract}

\maketitle
\vh{Suggestion for a shorter title: "Multi-Objective Approaches to Bounded-Parameter Markov Decision Processes"}
\ds{In principle, I would agree, but this title seems somewhat catchier to me.
It is also better Googleable.}
% Set this boolean variable to "true/false" for the long/short version.  
\setboolean{long}{false}   
%\ifthenelse{\boolean{long}}{}{}

%\setcounter{secnumdepth}{2}  
\section{Introduction}
\label{sec:Introduction}
\vh{I think the term ``uncertain MDPs'' used in the intro is not consistent with neither the title nor our defined model in the next section.}
\ifthenelse{\boolean{long}}{
Markov decision processes are a common tool to describe decision
situations in many different contexts, such as economy, artificial
intelligence and planning~\cite{Puterman94}. The general idea is to
specify a system by means of different \emph{states} in which it can
be, \emph{actions} which a decision maker can perform to affect the
(probabilistic) future behavior, and \emph{rewards} or \emph{costs} that
depend on the state and decision. After an action has been chosen, the
system changes its state depending on the action and the current state
but not on the past behavior; transitions are, in general, randomized
and defined by the system's properties.}
{
Markov decision processes (MDPs) are a common tool to describe decision
situations in many different contexts such as performance optimization
and planning~\cite{Puterman94,BeST16,BeM17,kolobov2012planning}. The general idea is to
specify a system by means of different \emph{states} in which it can
be, \emph{actions} which a decision maker can perform to affect the
(probabilistic) future behavior, and \emph{rewards} or \emph{costs} that
depend on the state and decision such as energy costs of a server starting up or
the amount of users served in a queue if a service is active.
After an action has been chosen, the system changes its state depending on the
action and the current state but not on the past behavior; transitions are, in
general, randomized and defined by the system's properties.
}
\ifthenelse{\boolean{long}}{
However, modeling a (physical or biological or artificial) system 
suffers from several limitations, one of the most important is the 
inherent loss of precision that is introduced by measurement errors and 
discretization artifacts which necessarily happen due to objective 
limitations of measuring equipment or due to incomplete knowledge about 
the system behavior. Thus, the real probability distribution for 
transitions is in most cases an uncertain value which is given by either 
external parameters or confidence intervals. For the latter, the Markov 
decision process model has been extended to bounded-parameter Markov 
decision processes (BMDPs)~\cite{DBLP:journals/ai/GivanLD00} or the 
slightly more general classes of MDPs with incomplete or uncertain 
transition probabilities~\cite{SaLa73,WhEl94}. In these classes of 
MDPs, best-case and worst-case policies for expected discounted reward  
measures have been considered. These policies result in upper and lower 
bounds for the value vector that contains the discounted accumulated 
reward gained after starting in the different states of the process. 
From a more abstract point of view, computing worst-case and best-case 
policies means solving the robust optimization problem for an uncertainty set
of MDPs.}

{However, modeling a physical or an artificial system suffers from several limitations. Most prominent is the 
inherent loss of precision that is introduced by measurement errors and 
discretization artifacts which necessarily happen due to incomplete knowledge 
about the system behavior. 
\vh{I find the following sentence obscure.}
\ds{Better?\vh{Not really. The Markovian assumption may hold even in case of modelling uncertainties and basically depends on the type of the properties.}}
\iffalse
Another possible source of uncertainty may be a violation of the Markovian
assumption. Incomplete modeling or aggregating different states~\cite{GiDG03}
of a system into one might introduce different behavior, depending on which
state the system has come from. 
\fi
So the true probability distribution to be 
associated with transitions is in most cases uncertain and instead can be given by either 
external parameters or confidence intervals. To account for the latter,
the MDP model has been extended to bounded-parameter MDPs (BMDPs)~\cite{DBLP:journals/ai/GivanLD00} or the slightly more general classes of MDPs with incomplete or uncertain 
transition probabilities~\cite{SaLa73,WhEl94}. In these classes of 
MDPs, best-case and worst-case policies for expected discounted reward  
measures have been considered. These policies result in upper and lower 
bounds for the value vector that contains the discounted accumulated 
reward gained after starting in the different states of the process. 
}
\ifthenelse{\boolean{long}}{
From the point of view of a potential decision maker, however, a robust 
solution that hedges against the worst possible realization of the 
uncertainty may be overly pessimistic. A decision maker is often 
interested in a solution that might be not fully robust, but instead 
could have an acceptable worst-case behavior while retaining good 
best-case and average-case (for a properly defined concept of 
``average'') performance. This property may be formalized in several 
alternative ways; most notably, there are several competing definitions 
of ``almost robust'' solutions~{\cite{KlaKST13}}. A further optimization 
goal in this context may be the probability of reaching a certain 
performance bound, under the assumption of some probability distribution 
over the uncertainty set. Unfortunately, the latter seems to be 
inherently hard~\cite{Sc15}.}

{Robust optimization techniques~\cite{DBLP:journals/mor/WiesemannKR13} are useful to derive
policies that hedge against model uncertainties. In particular, these
\emph{robust} policies 
optimize the expected discounted reward against the worst possible realization of the 
uncertainty. From the viewpoint of a potential decision maker, however, a robust 
solution may be overly pessimistic. A decision maker is often 
interested in a solution that might be not fully robust, but instead 
could have an acceptable worst-case behavior while retaining good 
best-case and average-case (for a properly defined concept of 
``average'') performance. This property may be formalized in several 
alternative ways; most notably, there are several competing definitions 
of ``almost robust'' 
solutions~\cite{KlaKST13}. 
A further optimization 
goal in this context may be the probability of reaching a certain 
performance bound, assuming some probability distribution 
over the uncertainty set. Unfortunately, the latter seems to be 
inherently hard, especially for uncertain MDPs~\cite{Sc15}.
}
\ifthenelse{\boolean{long}}{
Thus a promising approach considers not only one 
extreme performance measure (either the upper or the lower bound), but 
at least both bounds (and possibly the expected case, too) and 
optimize all of them simultaneously. This means doing 
\emph{multi-objective optimization} which yields several mutually 
incomparable, so-called \emph{non-dominated} policies for the uncertain 
MDP from which the user may choose the one which has the most suitable 
performance measures.

In this contribution, we develop methods to compute all pure 
non-dominated policies for a given uncertain MDP\@. Concretely, this is a 
problem from the domain of \emph{multi-scenario optimization} where one 
considers each policy in several \emph{scenarios}, i.\,e., problem 
instances that may arise. As we consider this as a multi-objective 
optimization problem, it is important to note that the number of 
optimal, i.\,e., mutually incomparable and non-dominated solutions may 
be exponentially high, which is a structural feature of the considered 
problem, and computing all of these solutions might take a large amount 
of time. However, for ``stationary'' decision-making problems that only 
rarely vary with time, we argue that the amount of time invested in 
finding an optimal policy is often negligible when compared to the 
actual performance of the policy implementation; a policy in a large system is
executed over the course of weeks, and investing several hours or even a day to
find that policy seems acceptable.
}

{
Thus a promising approach considers not only one 
extremal measure (either the upper or the lower bound), but 
at least both bounds, and possibly the expectation, and  
optimizes all of them simultaneously. This means doing 
\emph{multi-objective optimization} which yields several mutually 
incomparable, so-called \emph{non-dominated} policies for the uncertain 
MDP from which the user may select the one which has the most suitable 
performance measures. This may be in particular of interest if the differences
between the optimal policies in the different scenarios are sufficiently large
to consider different tradeoffs besides the obvious extremal points.
\vh{Check below!}
In this paper, we develop methods to compute all pure 
non-dominated policies for a given uncertain MDP in a specific uncertainty
setting which can be generalized to other notions of parameter uncertainty. This problem which is basically an 
instance of \emph{multi-scenario optimization} asks to compute policies which     
are optimal in the presence of trade-offs between several conflicting objectives.

It is worthwhile to mention that for the considered multi-objective optimization problem, 
the number of optimal, i.e., mutually incomparable and non-dominated solutions may 
be exponentially high, which is a structural feature of the problem, and computing all of 
these solutions might take a prohibitive amount of time. However, for ``stationary'' decision-making 
problems which only rarely vary with time, we argue that the amount of time invested in 
finding an optimal policy is often negligible when compared to the 
actual performance of the policy implementation.
}
%TODO Einf\"uhrung komplett neu schreiben.

% take one idea of \cite{Sc15}
% to assume a probability distribution over the uncertainty set and define 
% an ``average MDP'', and consider the average as well as the upper and 
% lower value bounds simultaneously. This reasoning is motivated by the 
% requirement to infer not only ``best-case'' and ``worst-case'' 
% performance but also to consider the performance on average and give the 
% user a possibility to pick a policy that suits her or his preference for 
% the best-case, worst-case and average-case performance simultaneously. 
% By doing so, we depart into the realm of \emph{multi-objective 
% optimization} where more than one solution may be optimal, since 
% solutions can be incomparable. This resonates with the ideas of 
% \cite{KlaKST13}, where also a general multi-objective approach is 
% applied to optimize various robustness measures.

\textbf{Related work.} 
\ifthenelse{\boolean{long}}{
For a general introduction on MDP theory and solution methods, we 
refer to the book~\cite{Puterman94}. Applications of MDPs and uncertainty
modeling in the performance optimization world are various and include decision
support in medical screening procedures~\cite{BeST16} and product line
design~\cite{BeM17}.
The extension to bounded-parameter 
Markov decision processes (BMDPs) is introduced and widely discussed 
in~\cite{DBLP:journals/ai/GivanLD00}. BMDPs are a specific subclass of 
MDPs with uncertain or imprecise transition rates proposed 
by~\cite{SaLa73} and~\cite{WhEl94}. MDPs with uncertainty have been 
extended even further, our results also apply to convex uncertainty sets 
discussed by~\cite{DBLP:conf/cav/PuggelliLSS13}. Until today, several 
aspects of MDPs with imprecise parameters are considered in the 
literature~\cite{DBLP:journals/ai/DelgadoSB11,DBLP:journals/mor/WiesemannKR13}.
However, in these cases, the goal 
is to compute some robust policy which assures the best possible 
behavior in the worst case.

We analyze MDPs with parameter uncertainty in the context of
multi-objective multi-scenario optimization which allows us to 
simultaneously optimize the worst, best and average case behavior. 
Applications and the general formulation of multi-scenario optimization 
are given in~\cite{Wiecek2009}. Multi-objective Markov decision 
processes are discussed in~\cite{WidJ07,White82,DBLP:conf/stacs/ChatterjeeMH06,DBLP:conf/icml/BarrettN08,DBLP:conf/ecai/PernyW10}. 
Furthermore, BMDPs can be considered a variant of stochastic games; the properties of stochastic games are 
discussed in the textbook~\cite{FiVr96}; the extension to 
multi-objective rewards in stochastic games is introduced 
by~\cite{DBLP:conf/mfcs/ChenFKSW13}.}
{
%For a general introduction on MDP theory and solution methods, we 
%recommend the book~\cite{Puterman94}. 
The Bounded-parameter MDP (BMDP) model is introduced and widely discussed 
in~\cite{DBLP:journals/ai/GivanLD00}. BMDPs are a specific subclass of 
MDPs with uncertain or imprecise transition probabilities proposed 
by~\cite{SaLa73} and~\cite{WhEl94}. \vh{The following sentence needs to be revised.}
The methods extend to more general notions of uncertainty in MDPs such as
convex uncertainty sets discussed by~\cite{DBLP:conf/cav/PuggelliLSS13}; further
aspects of parameter uncertainty in MDPs are covered
in~\cite{DBLP:journals/ai/DelgadoSB11,DBLP:journals/mor/WiesemannKR13,Iy05}.
\vh{Revised below!}
However, in almost all cases, the goal is to compute a robust policy 
%to satisfy a given objective 
which ensures the best possible behavior in the worst case.
In many scenarios however, we may have a multidimensional reward function and
hence search for a policy which simultaneously maximizes all reward dimensions.
To account for the latter, multi-objective Markov decision 
processes are discussed in~\cite{WidJ07,White82,DBLP:conf/stacs/ChatterjeeMH06,DBLP:conf/icml/BarrettN08,DBLP:conf/ecai/PernyW10}. 
The extension to multi-objective Markov decision processes under bounded
parameter uncertainty has recently been investigated in~\cite{HashemiQEST17}.
In this paper, we target another facet of multi-objective multi-scenario optimization for MDPs with parameter uncertainty where the goal is to 
simultaneously optimize the worst, best and average case behavior. 

%BMDPs can also be considered as a variant of stochastic games; the properties of stochastic games are 
%discussed in the textbook~\cite{FiVr96}; the extension to multi-objective rewards in stochastic games is introduced 
%by~\cite{DBLP:conf/mfcs/ChenFKSW13}.

There exist numerous applications of MDPs and uncertainty modeling in the
performance optimization world. They include decision support in medical
screening procedures~\cite{BeST16} and product line design~\cite{BeM17}. Applications and the general formulation of multi-scenario optimization 
are given in~\cite{Wiecek2009}. 
}

\textbf{Our Contribution.} 
\ifthenelse{\boolean{long}}{
This work considers bounded-parameter Markov decision processes as 
multi-scenario multi-objective optimization problems, where the 
individual scenarios correspond to the maximal, minimal, and average 
performance measures of a given policy in the uncertainty set given by 
the BMDP\@. The notion of average performance is given by slightly 
extending the existing BMDP formalism and defining a designated average 
MDP in the uncertainty set. Given this multi-objective optimization 
problem, we consider pure Pareto optimal policies and provide 
approaches to (heuristically) compute this set.

Our work differs from existing publications such as~\cite{WidJ07,White82,DBLP:conf/stacs/ChatterjeeMH06,DBLP:conf/icml/BarrettN08,DBLP:conf/ecai/PernyW10,DBLP:conf/mfcs/ChenFKSW13} 
on multi-objective Markov decision processes and stochastic games in one substantial 
aspect. In previous works, the multi-objective aspect was introduced by 
introducing additional reward components. In our case, the 
multi-objective aspect stems from the existence of several 
\emph{scenarios}, i.\,e., the upper and lower bounds for the 
performance under uncertainty. This especially means that not only the 
rewards contribute to the components of the objective function, but 
also the transition probabilities that maximize or minimize the 
performance metric. This problem seems to be substantially different, 
as the upper and lower reward bounds yield different Markov decision 
processes for which a single policy must be found. To the best of our 
knowledge, this perspective has not been considered in literature yet.

Our results here are two-fold. First, we provide an exact algorithm that 
computes the desired set of pure Pareto optimal policies, albeit at a 
(almost certainly) prohibitively high computational cost. Second, we 
design a heuristic that is efficient in the sense that it computes 
a set of mutually non-dominated policies that are likely to be Pareto 
optimal with reasonable time complexity. Third, we evaluate our 
heuristic: for small problem instances, we evaluate it against the 
exact algorithm, for larger problems, we compare its performance with a 
generic evolutionary multi-objective optimization method.}
{
\vh{Revised below! Please check.}
\ds{Looks good to me.}
\sloppy This work considers the multi-scenario multi-objective optimization problem for BMDPs  
where the multiple scenarios correspond to the maximal, minimal, and average 
performance measures of a given policy in the uncertainty set of BMDPs.
The notion of average performance is given by slightly 
extending the existing BMDP formalism and defining a designated average 
MDP in the uncertainty set. In particular, we define the average MDP by introducing a
probability distribution over the uncertainty set and deriving the expected
value. This has the advantage of possible further applications such as
percentile optimization.

In summary, the main contributions of this paper are threefold.
%\begin{itemize}
%	\itemsep0em 
First, we provide an exact algorithm that computes the desired set of pure Pareto optimal policies, 
	albeit at a (almost certainly) prohibitively high computational cost. 
Second, we design a heuristic that is efficient in the sense that it computes a set of mutually non-dominated policies 
	that are likely to be Pareto optimal with reasonable time complexity.
Finally, we develop a prototype tool and apply it to some case studies to show the effectiveness of our approach.
%\end{itemize}
}

\section{Mathematical Preliminaries}
\label{sec:Preliminaries}

%First we introduce common notations and formalisms used in this paper.
We start with common notations and formalisms used in this paper.
%\paragraph{Notation}
 
For a matrix $M$, we denote by $m_{i,j}$ the entry
in row $i$ and column $j$. Vector identifiers, such as $\vec v$, appear
in bold script, to distinguish them from scalars and matrices. For a
natural number $n$, we designate  by $[n]$ the set $\cbr{1, 2, \ldots,
	n}$. For multi-dimensional identifiers, such as matrices or vectors,
the order relations $\le$ or $\ge$ 
%mean ``less than or equal in all components'' or, respectively, ``greater than or equal in all
%components''.
are performed componentwise. Additionally, we define $\vec{v} >_P \vec{v}'$ to
be short for $\vec{v} \ge \vec{v}' \land \vec{v} \neq \vec{v}'$.
\ifthenelse{\boolean{long}}{%
  \input{preliminaries-long}
}%
{
%We begin with the definition of Markov reward processes which are the basic
%stochastic process type resulting from an MDP and a fixed policy. 

%\begin{definition}[Markov reward process]
%	A \emph{Markov reward process} (MRP) is a tuple $(S, P, \vec r)$,
%	where $S = [n]$ is the (finite) state space, $P \in \R^{n \times n}$ ($P
%	\ge 0$, $P \idvt = \idvt$) is a stochastic \emph{transition matrix} and
%	${\vec r} \in \R^{n,1}_{\ge 0}$ is a non-negative reward vector.
%\end{definition}

%For an initial state $s_1 \in S$, a MRP defines a sequence of random
%variables $\del{X_t}_{t \in \N}$ where $X_1 = s_1$, and $X_{t+1}$ for $t
%\in \N$ is selected subject to the given probability distribution $\Pr
%\sbr{ X_{t+1} = s \mid X_{t} = s' } = p_{s, s'}$. Furthermore, it
%defines a sequence of random variables $R_t$ ($t \in \N$) where $R_t =
%r_s$ for $X_t =s$. We consider here only MRPs and also MDPs with finite
%state spaces. 
%\vh{Please double check to make sure that everything is well-defined.}
We next motivate and formally introduce the modelling formalism considered in this paper.

%\begin{definition}[\textbf{Markov Decision Process}]
%	A \emph{Markov decision process} (MDP) is a tuple $(S, A, T, R)$ where  $S =
%	[n]$ is a (finite) set of states, $A = \sbr{m}$ is a (finite) set of
%	\emph{actions}, $R \colon S \times A \to \R_{\ge 0}$ is a set of $m$
%	reward vectors, and $T = \cbr{P^1, \ldots, P^m} \subset \R^{n \times
%		n}$ is a set of $m$ transition matrices with $P^a \ge 0$ and $P^a
%	\idvt = \idvt$ for all $a \in A$\footnote{It is straightforward to restrict
%		the  processes by allowing only a subset of action $A_s \subseteq A$
%		that can occur in state $s \in S$ by duplicating actions from
%	      $A_s$ to $A \setminus A_s$.}.
%\end{definition} 

%For a sequence of actions $\del{a^t}_{t \in \N}\in A$, an MDP defines
%sequences of random variables $\del{X_t}_{t \in \N}$ and $\del{R_t}_{t
%	\in \N}$ where $X_1 = s_1$ is the initial state, and $X_{t+1}$ for $t
%\in \N$ is chosen subject to the given probability distribution $\Pr
%\sbr{ X_{t+1} = s \mid X_{t} = s', a^t } = p^{a^t}_{s, s'}$; 
%furthermore, $R_t$ is defined as $R_t = R(X_t, a^t)$. 

%In the following discussion, we need individual rows of the transition
%matrix and define for 

%States of MRPs or MDPs often describe an aggregated view of the real 
%system such that the Markov property (i.e., the homogeneous and 
%memoryless transition probabilities) is only approximately correct and 

In a Markov decision process, transition probabilities and rewards are estimates
resulting from measurements or expert opinion. This implies that there is always
uncertainty about the parameters of the model and also about the behavior of the
real system according to some policy that has been derived from the MDP. The
class of stochastic bounded-parameter MDPs (SBMDPs) includes this uncertainty by
considering intervals rather than point estimates for the parameters of MDPs and
defining a probability distribution over the uncertainty set. We shall use the
probability distribution in order to derive the ``average'' MDP and then work
with this MDP; for the sake of generality, we would like to define the formalism here 
completely with the probability distribution.

%
%\begin{definition}[Bounded-parameter Markov decision process~\cite{DBLP:journals/ai/GivanLD00}]
%	A \emph{bounded-parameter Markov reward process} (BMRP) is a tuple $(S,
%	P_{\updownarrow}, {\vec r}_{\updownarrow})$ where $S=[n]$ is a (finite) set of
%	states, 
%	$P_{\updownarrow} = (P_{\downarrow},P_{\uparrow})$, 
%	$0 \le P_{\downarrow} \le P_{\uparrow}$,
%	$P_{\downarrow}\idvt \le \idvt \le P_{\uparrow}\idvt$, and ${\vec r} = ({\vec
%		r}_{\downarrow},{\vec r}_{\uparrow})$, $0 \le {\vec r}_{\downarrow} \le
%	{\vec r}_{\uparrow}$. 
%	A BMRP defines a set of MRPs 
%	$\left((S, P, {\vec r}) \mid P_{\downarrow} \le P \le P_{\uparrow}, 
%	P\idvt = \idvt, \vec r_{\downarrow} \le \vec r \le \vec 
%	r_{\uparrow}\right)$.   
%	%\thickmuskip=0.5\thickmuskip	
%	Analogously, a \emph{bounded-parameter Markov decision process} (BMDP) is a
%	tuple $(S, A, T_{\updownarrow}, R_{\updownarrow})$ where $S=[n]$ is a (finite)
%	set of states, $A = [m]$ is a (finite) set of actions, $R_{\updownarrow}
%	= \left(((\vec r_{\downarrow}^1, \vec r_{\uparrow}^1),\ldots, (\vec
%	r_{\downarrow}^m, \vec r_{\uparrow}^m)\right)$ is a set of $m$ reward
%	vector pairs, and $T_{\updownarrow} =
%	\left((P_{\downarrow}^1,P_{\uparrow}^1),\ldots,(P_{\downarrow}^m,P_{
%		\uparrow}^m)\right)$ is a set of $m$ matrix pairs. For each stationary
%	policy $\pi$, $(S, P_{\updownarrow}^{(\pi)}, \vec
%	r_{\updownarrow}^{(\pi)})$ is a BMRP\@. 
%\end{definition}

%SBMDPs define upper and lower bounds for the transition probabilities and
%rewards and allow one to analyze the worst and best case
%behavior. 

\begin{definition}[\textbf{Stochastic bounded-parameter Markov process}]
	\label{def:sbmdp}
	%	For a bounded-parameter Markov reward process $\del{ S, \bp{P},
	%		\bp{\vec r} }$ and a probability measure $Pr$ on $\bp{P},
	%	\bp{\vec r}$, 
	%	a \emph{stochastic bounded-parameter Markov reward
	%		process} (SBMRP) is the tuple $\del{ S, \bp{P}, \bp{\vec r}, Pr}$.
	%	By adding a probability measure on the transition matrices for each
	%	action and the rewards, 
	A \emph{stochastic bounded-parameter Markov
		decision process} (SBMDP) is a tuple $(S, A,  T_{\updownarrow}, 
	R_{\updownarrow}, Pr)$ where $S = [n]$ is a (finite) state space, $A = \sbr{m}$ is a (finite) set of
	\emph{actions}, $T_{\updownarrow} =
	((P_{\downarrow}^1,P_{\uparrow}^1),\ldots,(P_{\downarrow}^m,P_{
	\uparrow}^m))$ is a set of $m$ matrix pairs where for each $a\in A$, $0 \le P^a_{\downarrow} \le P^a_{\uparrow}$,
    $P^a_{\downarrow}\idvt \le \idvt \le P^a_{\uparrow}\idvt$, $R_{\updownarrow}
	= ((\vec r_{\downarrow}^1, \vec r_{\uparrow}^1),\ldots, (\vec
	r_{\downarrow}^m, \vec r_{\uparrow}^m))$ is a set of $m$ reward
	vector pairs where for each $a\in A$, $0 \le {\vec r}^a_{\downarrow} \le
        {\vec r}^a_{\uparrow}$ and $Pr = \cbr{(p_{a,r} \colon \R^n \to \R,
        p_{a,P} \colon R^{n,n} \to \R) \mid a \in A}$ is a probability measure
        on the rewards and the transition matrices for each
	action.
	%	$P \in \R^{n \times n}$ ($P
	%	\ge 0$, $P \idvt = \idvt$) is a stochastic \emph{transition matrix},
\end{definition}
We denote by $\vec p^a_{\updownarrow s}$  all vectors $\vec p_s^a$ such that
$p_{\downarrow s, s'}^a \le p_{s,s'}^a \le p_{\uparrow s, s'}^a$ for all
$s' \in [n]$ and $\sum_{s'=1}^n p_{s,s'}^a = 1$. Similarly $\vec r^a \in
\vec r_{\updownarrow}^a$ specifies all vectors $\vec r^a$ such that $
r_{\downarrow s}^a \le r_s^a \le r_{\uparrow s}^a$.  

It is worthwhile to note that the model of SBMDPs extends the formalism of BMDPs
introduced by~\cite{DBLP:journals/ai/GivanLD00} with a probability measure on
the possible transition matrices and reward vectors, so as to enable to take the
``average performance'' into consideration by deriving expected values for the
transition matrices and rewards. This is different from the expected value of
the value vector under the probability distribution, as here the expectation
operator is applied on the model and serves to define an additional optimization
objective, in addition to the upper and lower bounds for the value vectors.

To optimize the performance of a (SB)MDP, a \emph{decision rule} or \emph{policy}
is needed. Formally, a policy is a function $f \colon S^\N \to
\Delta(A)$ where $S^\N$ is the set of (finite) \emph{histories} of states and
$\Delta(A)$ is the set of probability distributions on $A$.
We call a policy $f$ \emph{pure} if it is
stationary, i.e., it depends only on the current state, and deterministic, i.e,
it always maps a history to a Dirac distribution, i.e., $f(\cdot, a) \in
\cbr{0,1}$.
%\begin{itemize}
%	\item \emph{stationary} if it depends only on the current state,
%	\item \emph{deterministic} if it always maps a history to a Dirac
%	distribution, i.e., $f(\cdot, a) \in \cbr{0,1}$,
%	\item \emph{pure} if it is stationary and deterministic,
%	\item \emph{mixed} if it is stationary, but not pure.
%\end{itemize}
We use $f$ for general policies and $\pi$ for stationary and pure policies. 
%A pure policy can be described by a vector of length $n$ including in position
%$s$ the action chosen in state $s$, or, alternatively, as a function 
%with domain $S$ and codomain $A$. 
We denote by $\pi(s) \in A$ the action that is chosen in state $s$ using policy
$\pi$, if $\pi$ is pure. 

Moreover, a stationary policy $\pi$ induces a Markov reward process with 
transition matrix $P^{(\pi)}$ and reward vector $\vec r^{(\pi)}$. 
%as the transition probabilities under $\pi$ depend only on the current state. 
For a transition matrix $P$ and state $s$, we denote by $\vec p_s$ the $s$-th row of $P$.
Accordingly, we denote by 
%\vh{Better to use $\vec p_s^{f(s)}$ instead of $\vec
%p_s^{\pi_s}$ for the sake of consistency.} 
$\vec p_s^{\pi(s)}$ the $s$-th row of
$P^{(\pi)}$ and by $r_s^{\pi(s)}$ the reward of state $s$ under policy $\pi$. For
pure policies, we define a Hamming distance measure $d$
%$d \colon A^S \times A^S \to \cbr{0, \ldots, n}$. It is 
by
$d(\pi, \pi') = |{ \cbr{s \in S \mid \pi(s)
\neq \pi'(s)} }|$.

The \emph{expected discounted reward} for an infinite horizon is taken as 
performance metric. 
%For an arbitrary sequence of
%rewards $\del{r_t}_{t \in \N}$ and a \emph{discount factor} $\gamma \in
%\intco{0,1}$, the sum $\rho = \sum_{t \in \N} \gamma^{t-1} r_t$ defines the
%discounted reward. Since the rewards generated by a given Markov decision
%process are bounded and the discount factor is smaller than $1$, the sum
%converges. % to a finite value. 
For a MDP, a given policy
$f$ defines a sequence of transition matrices $(P^{(f)}_t)_{t \in
\N}$ and reward vectors $(\vec r^{(f)}_t)_{t \in \N}$ where $P^{(f)}_t$
resp.\ $\vec r^{(f)}_t$ is the transition matrix resp.\ the reward in the
$t$-th time step. Using this sequence, we derive a probability
distribution on sequences of states and corresponding rewards, from
which the expected discounted reward is defined by 
$\Ex \sum_{t \in \N} \gamma^{t - 1} \vec{r}^{(f)}_t$.
%\[ 
%  \Ex \sum_{t \in \N} \del{\gamma^{t-1} \vec r^{(f)}_t } 
%  = \sum_{t \in \N} \gamma^{t-1}\Ex \sbr{\vec r_t^{(f)}}
%  .
%\]
%where the linearity of the expectation operator is used to decompose the
%computation.

A \emph{value vector} $\vec v$ collects the computed expected discounted reward for all states where each entry 
equals the expected discounted reward one obtains starting from the respective state.
The mathematical properties of expected discounted rewards on which we
shall rely in the sequel are widely discussed by~\cite{Puterman94}.

We assume that the goal is the maximization of the expected discounted
reward. 
For a SBMDP policies assuring best and worst case behavior can be
determined by solving the following Bellman equations for the
\emph{value vectors} $\vec v_\downarrow, \vec v_\uparrow$.
\begin{equation}
%  \footnotesize
\label{eq:upper-lower-bound}
\begin{array}{llll}
v_{\downarrow s} & = & 
\max_{a \in A} \left(r_{\downarrow s}^a + \gamma
\min_{\vec p_{s}^a \in  \vec p^a_{\updownarrow s}} \left(
\vec p_{s}^a \vec v_{\downarrow}  \right) \right)\\
v_{\uparrow s} & = & 
\max_{a \in A} \left(r_{\uparrow s}^a + \gamma
\max_{\vec p_{s}^a \in \vec p^a_{\updownarrow s}} \left(
\vec p_{s}^a \vec v_{\uparrow}  \right) \right)
\end{array}
\end{equation}
As shown by~\cite{DBLP:journals/ai/GivanLD00}, 
%the minimum and maximum
%inside the equations can be solved easily for BMDPs such that the
%computation of the vectors $\vec v_{\downarrow}$ and $\vec v_{\uparrow}$
%is a MDP-like problem which can be solved with value iteration or 
%policy iteration.
the optimal policies are pure. We denote them by $\pi_{\downarrow}$ and $\pi_{\uparrow}$, respectively.  

Additionally, we are interested in the ``average-case'' properties of 
SBMDPs and a policy that maximizes the performance under the ``average-case'' assumption. 
The reasoning here is fairly straightforward: Since the probability distribution for 
the transition matrices is known, we can define for each action $a \in 
A$, $\overline{P}^a = \Ex \sbr{ P^a } = \int_{P^a_{\updownarrow}}
P \cdot p_{a,P}(P) \dif P$ and $\overline {\vec r}^a = \Ex[\vec r^a] =
\int_{{\vec r}^a_{\updownarrow}} {\vec r} \cdot p_{a,r}({\vec r}) \dif
{\vec{r}}$, where $p_a$ are probability densities on transition matrices and
reward vectors for
action $a$. The Bellman equations become
\begin{equation}
%  \footnotesize
\label{eq:average}
\begin{array}{lll}
\overline {v}_s & = & 
\max_{a \in A} \left(\overline{r}_{s}^a + \gamma
{\overline {\vec p}}_{s}^a \overline {\vec v}\right)
\end{array}
\end{equation}
which is a standard MDP problem. The optimal policy is pure and is
denoted as $\overline{\pi}$. For a policy $f$, by
$\vec{v}_{\downarrow}^{(f)}$, $\vec v_{\uparrow}^{(f)}$ and $\overline
{\vec v}^{(f)}$ we designate the lower, upper and average value vectors,
furthermore, the triple $(\vec{v}_{\downarrow}^{(f)},
\overline{\vec{v}}^{(f)}, \vec{v}_{\uparrow}^{(f)})$ is designated by
$\vec{v}^{(f)}$, if not mentioned otherwise. Obviously
$\vec{v}_{\downarrow}^{(f)} \le \overline {\vec v}^{(f)} \le
\vec{v}_{\uparrow}^{(f)}$ and also $\vec{v}_{\downarrow} \le
\overline{\vec v} \le \vec v_{\uparrow}$ hold, as $\vec{v}_{\downarrow}^{(f)}$
and $\vec{v}_{\uparrow}^{(f)}$ are the minimal resp. maximal value vectors for
$f$ over all MDPs in the uncertainty set. Usually %the policies
$\pi_{\downarrow}$, $\pi_{\uparrow}$, and $\overline \pi$ differ and a
compromise between two or all three goals has to be found. 
%Before we introduce different approaches to handle the problem, a few remarks have
%to be made.

Before we discuss specifics, we would like to mention a possible alternative
approach in this setting. It is possible to consider the SBMDP formalism
in a Bayesian setting, where the information from the probability distribution
over the uncertainty set of MDPs is taken into account completely, thus seeking
for $\max_{f} \Ex_{M} \vec{v}_{M}^{(f)}$ where $\vec{v}_{M}^{(f)}$ is the value
vector of the policy $f$ in the MDP $M$. This approach can be undertaken in the
case where it is possible to integrate over the probability distribution over
$\bp{M}$ easily.

Here, we explicitly choose the multi-objective approach mainly for the reason
of larger applicability. Even if a Bayesian approach can be motivated from the
standpoint of stochastic BMDPs, one can easily see that any reasonably good
solution in the multi-objective setting can be transferred to ``plain'' BMDPs and, furthermore, to
multi-objective MDPs even without introducing a probability distribution over
the uncertainty set.\looseness-1

%\paragraph{Extended models}
%
%The BMDP framework is not the only way to describe parameter uncertainty
%in Markov decision processes. An alternative is to parameterize the
%uncertainty sets of transition matrices with variables $\lambda_1,
%\ldots, \lambda_k$ such that for each state $s$ and each action $a$,
%there is a function $h_{s,a} \colon \R^k \to \Rn$ that maps the
%parameter space to a set of stochastic vectors. By restricting the
%functions $h_{s,a}$ to simple classes, such as linear or affine
%functions, one can handle this perspective in a similar fashion to the
%one presented in this paper although the computation of the  minimum or
%maximum in (\ref{eq:upper-lower-bound}) might be more complex and
%requires in extreme cases the solution of an LP problem itself, or,
%worse, can be rendered \NP-hard~\cite{Sc15}.
%
%\paragraph{Continuous time models}
%
%We consider here MDPs in discrete time. MDPs in continuous time can be
%transformed into discrete time MDPs using uniformization
%\cite{Serfozo79}. The resulting discrete time MDP shows an equivalent
%behavior with respect to the discounted reward. In a similar way, it is
%possible to transform BMDPs in continuous time into discrete time BMDPs
%using uniformization. We will not formalize the approach here but keep
%in mind that continuous time models can be handled as well with the
%following approaches. 

%\paragraph{Pareto optimality}
As we shall consider multi-objective optimization problems, we define 
the set of solutions that interest us.
\begin{definition}[\textbf{Pareto Optimality}]
\label{def:ParetoOptimality}
Let $\Ppure$ be the set of pure policies $\pi \colon S \to A$, then we define
the \emph{Pareto frontier} as
\[
%  \footnotesize
\begin{array}{l}
\Ppareto = \Big\{\pi \in \Ppure \mid
\not\exists \pi' \in \Ppure: %\\
%\vec v_{\downarrow}^{(\pi)} \le \vec v_{\downarrow}^{(\pi')} \land 
%\overline{\vec v}^{(\pi)} \le \overline{\vec v}^{(\pi')} \land 
%\vec v_{\uparrow}^{(\pi)} \le \vec v_{\uparrow}^{(\pi')} \land
\vec{v}^{(\pi')} >_P \vec{v}^{(\pi)} 
%\land \vec{v}^{(\pi)} \neq \vec{v}^{(\pi')}
%\\\left(
%\vec v_{\downarrow}^{\pi} \neq \vec v_{\downarrow}^{\pi'} \vee
%\overline{\vec v}^{\pi} \neq \overline{\vec v}^{\pi'} \vee
%\vec v_{\uparrow}^{\pi} \neq \vec v_{\uparrow}^{\pi'}
%\right)
\Big\}
\end{array}
\]
the set of pure policies that result in Pareto optimal value vectors. We
denote by $\Vpareto$ and $\Vpure$ the corresponding sets of value 
vectors.
\end{definition}
It is easy to see that %\Pweight \subseteq 
$\Ppareto \subseteq \Ppure$ holds and usually, the subsets are proper. For
convenience, we shall use an operator $PO(\cdot)$ that
operates on sets with a partial order $\ge_P$ and is defined by $PO(X) = \cbr{x
\in X \mid \neg\exists y \in X : y >_P x}$.

}

%\begin{comment}
\iffalse
  \input{theoretical-considerations}
\fi
\section{Computation of Optimal Policies}
\label{sec:algo}
\ifthenelse{\boolean{long}}{
  \input{solution-approach-long}
}%
%%%%%%%%%%%%%%%%% Revised short version is below. %%%%%%%%%%%%
{
%\vh{Revised. Please check!}
The computation of Pareto optimal policies in multi-objective MDPs is 
challenging as the number of optimal policies can be large or even
infinite (for non-stationary policies). Therefore, algorithms are required to approximate the 
Pareto set efficiently instead of computing the whole set of Pareto 
optimal policies. Several works have already explored methods to compute or 
approximate the Pareto set, e.g.,~\cite{DBLP:conf/ecai/PernyW10,DBLP:conf/icml/BarrettN08,DBLP:conf/stacs/ChatterjeeMH06,White82,WidJ07,Roijers:2014:LSM:2615731.2617454}
albeit for the computation of several expected values or in a setting that 
has only a multi-objective reward. In SBMDPs we have to face 
optimization of, ultimately, several MDPs with related but not 
identical transition probabilities; furthermore, we consider the worst 
and best case vectors. This seems to be a different, and generally harder problem. 

For most non-trivial models the number of Pareto optimal policies is 
still much too large to compute them all and many policies show a 
similar behavior. From a practical point of view it is sufficient 
to compute a subset of the Pareto frontier if the corresponding 
value vectors are evenly distributed in the value vector space. 
Theoretically, one can try value iteration or policy iteration-based 
approaches to compute policies from $\Ppareto$ and the corresponding 
value vectors from $\Vpareto$, as in~\cite{WidJ07}. The disadvantage of 
value iteration algorithms is that the number of intermediate 
partly evaluated policies can become enormous before the value iteration 
converges and the policies get completely evaluated. In this case, one 
has to stop the algorithm with a large number of non-comparable policies 
which might even not be optimal since value iteration provided only 
an approximation or a lower bound of the true value vectors. Thus, we 
consider a policy iteration-based approach.

In the following, we present two algorithms. The first algorithm 
computes $\Ppareto$ exactly, albeit at a possibly high computational 
cost. The second algorithm is a fixed-point iteration that computes an 
increasing set of mutually non-dominating policies. Our first approach 
relies on the following result. 
\ifthenelse{\boolean{final}}{
  The proofs of this and the following results can be found in the online
  companion~\cite{SBHH17online}.
}{}

\begin{restatable}{lemma}{lemmaNondominatedPath}
	\label{lma:non-dominated-path}
	Let $\mathcal{P} = (S, A, T_{\updownarrow}, R_{\updownarrow}, Pr)$ be 
	a SBMDP\@.
	%and let $\vec{v}^{\pi} > \vec{v}^{\pi'}$ hold iff the value 
	%vector resulting from the evaluation of policy $\pi$ dominates the 
	%value vector of policy $\pi'$ in all states and all components. 
	Let furthermore $\pi, \pi'$ be two policies where $\pi'$ lies on the
	Pareto frontier. Then there exists a finite sequence of policies $\pi =
	\pi_0, \pi_1, \ldots, \pi_N = \pi'$ where $d(\pi_i, \pi_{i+1}) = 1,
	\vec{v}^{(\pi_i)} \not> \vec{v}^{(\pi_{i+1})}$ and, additionally, $N \le 
	|S|$.
\end{restatable}
%\vh{The second part is a bit unclear.\\
%It is important to note that this result does not state anything about
%policies that are non-adjacent in the sequence $\pi_0, \ldots, \pi_N$, e.\,g.,
%it does not follow that $\pi_{i} \not> \pi_{i+2}$.}
%\ds{We can skip the comment. It was meant for clarification, and at some point I
%repeat the statement.}

From this result, we develop a procedure that computes $\Ppareto$ as sketched in
Algorithm~\ref{alg:ppareto-bfs}. Intuitively, the algorithm generates for each
currently considered policy possible non-dominated neighbor policies and updates
a global set of non-dominated policies. This continues for $|S|$ steps.

\begin{algorithm}[!t]
	\small
	\caption{Exact computation of $\Ppareto$ and $\Vpareto$}
	\label{alg:ppareto-bfs}
	\begin{algorithmic}[1]
		\Function{Pure-opt-exact\;}{$\mathcal{P} = (S, A, \bp{T}, \bp{R}, Pr), 
			\gamma$}
		\State $P \gets \cbr{\text{arbitrary policy}}$ \Comment
		initialize current policy set 
		\State $F \gets P$ \Comment initialize the Pareto frontier
		\For{$i \in \cbr{1, \ldots, |S|}$}
		\State $P \gets \cup_{\pi \in P} \cbr{ \pi' \mid d(\pi, \pi') = 
			1, \vec{v}^{\pi} \not> \vec{v}^{\pi'} }$
		\State $F \gets PO(F \cup P)$\label{alg:ppareto-bfs-updatefrontier}
		\EndFor
		\State \Return{$F$}
		\EndFunction
	\end{algorithmic}
\end{algorithm}
\begin{restatable}{theorem}{CorrectnessExactAlgorithm}
\label{lem:CorrectnessExactAlgorithm}
Algorithm~\ref{alg:ppareto-bfs} correctly computes $\Ppareto$. 	
\end{restatable}

The complexity of the exact algorithm is, as of yet, unclear. In the worst
possible case
the algorithm will produce large numbers of temporarily optimal policies
that will be dominated by a few Pareto-optimal policies in the end,
making the worst-case complexity $\bigO{m^n}$, theoretically independent
of the final size of $\Ppareto$. To avoid this, we propose a slightly
different algorithm that computes a set of policies that seem to be a
reasonably good approximation of $\Ppareto$. It is important to note
that the following approach is an heuristic; later, we shall discuss the
(empirical) quality of this heuristic.

The heuristic approach is a simplification of the
Algorithm~\ref{alg:ppareto-bfs}. The simplification lies in including only those
policies into the non-dominated set that are non-dominated with respect to all
other policies.
%in
%lines~\ref{alg:pure-opt-search-start}--\ref{alg:pure-opt-search-end}.
%\ds{Is this sufficient?  This seems to be the most important question for this
%  heuristic. I have not found any counterexamples (yet), although a formal proof
%  seems also nontrivial, as there are counterexamples to ``simple'' statements
%  such as ``there always exists a path that improves one component from any
%  policy''
%}
This decreases the number of candidate policies and, thus, the runtime at the
cost of possible inaccuracy. 

The practical implementation tries to find quickly policies that are 
spread over the Pareto frontier by analyzing first those policies which 
are likely to belong to the Pareto frontier. 
It is known that the policies $\pi_{\downarrow}, \overline{\pi}, \pi_{\uparrow}$
belong to $\Ppareto$, as no other policy has a greater value vector in the lower
bound, average, and upper bound case, respectively, by definition. 
Starting with these policies makes the algorithm walk through the 
policy space from the extreme points of the Pareto frontier which, as we 
hope, yields an evenly distributed non-dominated set of policies.

The following simple observation helps to find good policies which are 
candidates for the Pareto frontier without too many policy evaluations. 
Let for a policy $\pi$ and a state-action pair $(s, a)$ the \emph{gradient} be
\begin{equation}
%\footnotesize
\begin{aligned}
& \grad(\pi, s,a) = \left(
r^a_{\downarrow s} + \gamma \min_{\mat{p}_s^a \in \mat{p}_{\updownarrow s}^a}
\mat{p}_s^a \mat{v}_{\downarrow}^{(\pi)} - v_{\downarrow s}^{(\pi)}, \right. \\ & \left.
\overline{r}^a_s + \gamma \overline{\mat{p}}_s^a
\overline{\mat{v}}^{(\pi)} - \overline{v}_s^{(\pi)},
r^a_{\uparrow s} + \gamma \max_{\mat{p}_s^a \in \mat{p}_{\updownarrow s}^a}
\mat{p}_s^a \mat{v}_{\uparrow}^{(\pi)} - v_{\uparrow s}^{(\pi)}
\right).
\end{aligned}
\end{equation} 
If, for some policy 
$\pi$, 
a state $s$ and an action $a \in A \setminus \{\pi(s)\}$ can be found 
such that 
\begin{equation}
%\footnotesize
\label{eq:dominates}
  \grad(\pi, s, a) >_P \vec{0}
%\begin{aligned}
%\left(r^a_{s \downarrow} + \gamma \min_{\mat{p}_s^a \in \mat{p}_{
%  \updownarrow s}^a} \mat{p}_s^a \mat{v}_{\downarrow}^{(\pi)}, 
%\overline{r}^a_s + \gamma \overline{\mat{p}}_s^a
%\overline{\mat{v}}^a, r^a_{\uparrow s} + \gamma \max_{\mat{p}_s^a \in \mat{p}_{
%\updownarrow s}^a} \mat{p}_s^a \mat{v}_{\uparrow}^{(\pi)}
%\right)\\ >_P \left( v_{\downarrow s}^{(\pi)}, \overline{v}_s^{(\pi)},
%v_{\uparrow s}^{(\pi)} \right),&
%\end{aligned}
\end{equation}
then a policy $\pi^{(s,a)}$ can be defined with $\pi^{(s,a)}(t) = 
\pi(t)$ for $t \neq s$ and $\pi^{(s,a)}(s) = a$ and for 
$%\left(
  \pi^{(s,a)}$
%, \mat{u}_{\downarrow}, \overline{\mat{u}},
%\mat{u}_{\uparrow}\right)$ 
it is
$\vec{v}^{(\pi^{(s, a)})} >_P \vec{v}^{(\pi)}$.
%$\left(\mat{u}_{\downarrow}, 
%\overline{\mat{u}},\mat{u}_{\uparrow}\right) >_P 
%\left(\mat{v}_{\downarrow}, \overline{\mat{v}}, 
%\mat{v}_{\uparrow}\right)$. 
We define an operator 
\begin{equation*}
  %\footnotesize
\pi' =
popt\left(\pi, 
  \vec{v}^{(\pi)}
%  \mat{v}_{\downarrow},
%  \overline{\mat{v}},
%  \mat{v}_{\uparrow}
\right)
\end{equation*}
that generates a new policy from $\pi$ by selecting for each $s$ the
lexicographically smallest
action $a$ that observes~\eqref{eq:dominates}, whenever this is 
possible. Looping $popt$ until convergence, which occurs in the repeat-loop in
lines 14--17, yields a locally optimal policy in the Pareto frontier.
%Additionally, we introduce the operator $popt^*(\mathcal{P}, \pi)$
%which computes a fixed point of $popt$ in the SBMDP $\mathcal{P}$ starting from
%policy $\pi$. 

%For a policy $\pi$ 
%$\left(\mat{v}_{\downarrow},
%\overline{\mat{v}},\mat{v}_{\uparrow}\right)$ 
%and a state-action pair
%$(s,a)$ the gradient is defined as 
The function $popt$ considers only pairs $(s,a)$ where the gradient is
positive in all components, making $\pi^{(s,a)}$ dominate $\pi$. In fact, the
gradient gives a hint how the new policy will behave. From the gradient, the 
direction of the different value vectors of a new policy can be 
estimated without evaluating it.
%
%The function $popt$ considers only pairs $(s,a)$ where the gradient is
%non-negative and non-zero. If the gradient contains no 
%positive elements, then the corresponding policy $\pi^{(s,a)}$ is 
%dominated by $\pi$. If the gradient is non-negative and non-zero, then 
%$\pi^{(s,a)}$ dominates $\pi$. In all other cases some components of the 
%value vectors are improved and other will become worse. The gradient 
%gives a hint how the new policy will behave. From the gradient, the 
%direction of the different value vectors of a new policy can be 
%estimated without evaluating it.
Policy evaluation is done by a function {\em eval\/} which solves the
following three sets of equations to compute the value vectors for some pure
policy $\pi$.
\begin{equation}
  %\footnotesize
\begin{aligned}
\mat{u}_{\downarrow} = \mat{r}_{\downarrow}^{(\pi)} + \gamma
\min_{P \in P_{\updownarrow}^{(\pi)}} \left(P
\mat{u}_{\downarrow}\right), \,
\overline{\mat{u}} = \overline{\mat{r}}^{(\pi)} + \gamma
\overline{P}^{(\pi)} \overline{\mat{u}},\\ 
\mbox{and\;}\mat{u}_{\uparrow} = \mat{r}_{\uparrow}^{(\pi)} + \gamma
\max_{P \in P_{\updownarrow}^{(\pi)}} \left(P\mat{u}_{\uparrow}\right)
\end{aligned}
\end{equation}
The equations for the average values define a set of linear equations 
which can be solved with standard techniques. For the vector of the worst and 
best case an iterative approach is applied: The vectors are successively
computed by solving the LP and optimizing the matrix until
convergence~\cite{DBLP:journals/ai/GivanLD00}.

%The vector 
%$\mat{u}_{\downarrow}$ (or $\mat{u}_{\uparrow}$) is initialized, the 
%minimum (or maximum) is computed and with this minimum (or maximum), a 
%new vector is computed which is then used to find a new minimum (or 
%maximum). This procedure defines a sequence of value vectors that 
%converges to a unique fixed point~\cite{DBLP:journals/ai/GivanLD00}.

In our algorithm we use a set $PV$ that contains tuples $(\pi,
\mat{v}^{(\pi)})$. 
Additionally, we use a set $P$ where all evaluated policies are stored 
to avoid re-evaluations of policies. 
%Algorithm~\ref{alg:ivi-pure-popt} is an optimized version of the 
%policy iteration approach to 
%compute a heuristic approximation of $\Ppareto$ and the corresponding 
%value vectors. 
In the current description, new policies are generated 
starting from available policies by maximizing one direction of the 
gradient. 
%Other, more sophisticated strategies to generate new points in 
%the Pareto frontier using gradient information can be added. 
The algorithm stops if in the current set of non-dominated policies, the neighbors of 
each policy are either explored or dominated. A further stopping 
condition is to explore a predefined number of policies.

\begin{algorithm}[!t]
  \small
  \caption{An heuristic for $\Ppareto$ and $\Vpareto$}
  \label{alg:ivi-pure-popt}
  \begin{algorithmic}[1]
    \Function{Pure-opt}{$\mathcal{P} = (S, A, T_{\updownarrow},
    R_{\updownarrow}, Pr)$, $\gamma$} 
      \State $PV \gets \{(\pi_{\downarrow}, \vec{v}_{\downarrow}),
      (\overline{\pi}, \overline{\vec v}),
      (\pi_{\uparrow}, \vec{v}_{\uparrow})\}$;
    %$\mat{v}_{\downarrow}^{lb},\overline{\mat{v}}^{lb}$,
    %$\mat{v}_{\uparrow}^{lb}), (\pi^{av}$,
    %$\mat{v}_{\downarrow}^{av}, \overline{\mat{v}}^{av}$,
    %$\mat{v}_{\uparrow}^{av}), (\pi^{ub}$,
    %$\mat{v}_{\downarrow}^{ub}, \overline{\mat{v}}^{ub}$,
    %$\mat{v}_{\uparrow}^{ub})\}$ ;
      \State $P \gets \{\pi_{\downarrow}, \overline{\pi},
      \pi_{\uparrow}\}$ ;
      \While{$|PV| <$ max.\ number of policies}
        \label{alg:pure-opt-search-start}
        \State $\Pi = \emptyset$ ;
    %\State $u = \bot$; \Comment $u$ tells if new non-dominated 
    %policies have been found
        \For {$\pi \in PV$ and all $(s,a)$ where $\pi^{(s,a)} \notin P$}
          \State $(g^{(\pi,s,a)}_\downarrow, \overline{g}^{(\pi,s,a)},
          g^{(\pi,s,a)}_\uparrow) = \grad(\pi, s,a)$ ; $\Pi = \Pi \cup
          \left\{(\pi,s,a\right)\}$ ;
          \vh{changed above.}
        \EndFor
        \If {no non-negative gradient exists for $(\pi,s,a) \in \Pi$}
          \State break (all locally non-dominated solutions found);
        \EndIf
        \For{$g \in \cbr{g_{\downarrow}, \overline{g}, g_{\uparrow}}$}
          \If{$g^{(\pi, s, a)} > 0$ exists}
            \State $(\pi, s, a) \gets \arg\max_{(\pi, s, a) \in \Pi} \cbr{g^{(\pi, s, a)}}$ ;
            \State $\pi' \gets \pi^{(s, a)}$;
            \Repeat \Comment compute local optimum
            \State $\mat{v}^{(\pi')} = (\mat{v}_{\downarrow}, \overline{\mat{v}},
            \mat{v}_{\uparrow}) \gets eval(\mathcal{P}, \pi')$ ;
             \State $\pi' \gets popt(\pi', \vec{v}^{(\pi')})$;
            \Until{$\pi'$ does not change}
            %\State $\pi' \gets popt^* \del{ \mathcal{P}, \pi^{(s, a)}}$ ;
            %\vh{$popt^*$ is not defined.}
            \State $PV \gets PO\left(PV \cup \{(\pi', \mat{v}^{(\pi')})\}\right)$ ; 
            $P \gets P \cup \{\pi'\}$ ;
          \EndIf
        \EndFor
        \If{$PV$ was not changed}
          \State break (all new policies are explored or dominated)
        \EndIf
      \EndWhile \label{alg:pure-opt-search-end}
      \State \Return $PV$
    \EndFunction
  \end{algorithmic}
\end{algorithm}

As the algorithm is a heuristic, it is difficult to argue about 
guaranteed performance in terms of quality of the output, that is, if 
the resulting policies $P_r = \cbr{ \pi \mid \del{\pi, 
\cdot} \in PV}$ fulfill $\Ppareto \subseteq P_r$ and $P_r \subseteq 
\Ppareto$. It is still possible to make several observations. The main 
difference between Alg.~\ref{alg:ivi-pure-popt} and 
Alg.~\ref{alg:ppareto-bfs} lies in the bookkeeping that disallows 
Alg.~\ref{alg:ivi-pure-popt} to explore policies that have been 
already evaluated and, more importantly, are dominated by some other 
policy that has been found. By doing this, 
Alg.~\ref{alg:ivi-pure-popt} may ignore policies that are dominated 
yet lead (by choosing an appropriate sequence of policy changes) to the 
Pareto frontier; if there are no other ways to the Pareto frontier, this 
makes the output of Alg.~\ref{alg:ivi-pure-popt} incomplete. On 
the other hand, one can provide a heuristic argument: Since the initial 
set of policies contains known ``extreme points'' $\pi_{\downarrow}, \overline{\pi}, 
\pi_{\uparrow}$, the policies that will be found by 
Alg.~\ref{alg:ivi-pure-popt} in realistic settings will stem from a 
gradual transition from one extreme policy to another, as there will 
always exist a path of policies that improves one of the objectives until an
optimum is reached. This way, we can expect that in real-life 
problems, the resulting set $P_r$ will cover the Pareto frontier or at 
least the space between the value vector tuples $\vec{v}_\downarrow, 
\overline{\vec v}, \vec{v}_\uparrow$
%(\vec{v}_\downarrow^{av}, \overline{\vec v}^{av}, 
%\vec{v}_\uparrow^{av}), (\vec{v}_\downarrow^{ub}, \overline{\vec 
%	v}^{ub}, \vec{v}_\uparrow^{ub})$ 
adequately, i.\,e., the resulting set 
of value vectors will be evenly distributed in the space between the 
extreme value vectors stemming from $\pi_{\downarrow}, \overline{\pi},
\pi_{\uparrow}$. 
Furthermore, we expect that for practical problems, the following 
assumption will hold: If for a set of policies $P$ it is $P \subseteq 
\Ppareto$ and $P$ contains not all Pareto optimal policies, then there 
exists a policy $\pi \in P$ and a state-action pair $(s, a) \in S \times 
A$ such that $\pi^{(s, a)} \not\in P$ and $\pi^{(s, a)}$ is not 
dominated by any other policy in $P$. This especially means that there 
is always a ``globally'' non-dominated path of policies from a set $P$ 
of mutually non-dominating policies to a policy in $\Ppareto \setminus 
P$ if $\Ppareto \neq P$. We expect that this assumption holds for 
practical instances. Furthermore, we conjecture that our assumption is 
also true for the problem in general. 
% On the other hand, even if the assumption is not true, and $P_r 
% \not\subseteq \Ppareto$, then for any policy $\pi \in P_r, \pi \not\in 
% \Ppareto$ every neighbor is dominated by some other policy in $P_r$, 
% which means that 
Finally, since Algorithm~\ref{alg:ivi-pure-popt} is a variation of Algorithm~\ref{alg:ppareto-bfs}, the same complexity reasoning
applies here. However, the practical complexity should be lower than that of
Algorithm~\ref{alg:ppareto-bfs}, as less SBMDP evaluations have to be performed.

%As the algorithm computes successive sets of non-dominated policies, 
%one can also use these intermediate results as an approximation of the 
%final set $PV^*$; furthermore, it is also possible to decide, using an 
%intermediate set of non-dominated policies, if the algorithm has 
%terminated, and, if not, compute a further non-dominated policy with 
%$\bigO{|PV| |S| |A|}$ BMDP evaluations, if $PV$ is the currently 
%computed set of non-dominated policies.
}

\section{Evaluation}
\label{sec:evaluation}
\vh{ToDo: Add a brief outline about two case studies.}
We have performed a series of experiments where we considered several questions.
First, we compared the performance of Algorithm~\ref{alg:ivi-pure-popt} against
a black-box multi-objective optimization method as reference. We have chosen
\emph{SPEA2}~\cite{ZLT01} as reference since it is a well-studied, simple
black-box optimization algorithm. Second, we have evaluated the
general performance of the algorithm with respect to problem size and the number
of computed solutions.

For the evaluation, we have used a machine with an 
Intel\textregistered{} Core\texttrademark{} i7-4790 CPU and \num{16} \si{GB} 
RAM\@. We have set a time limit of \num{1000} \si{s} for \emph{SPEA2} and 
a limit for Algorithm~\ref{alg:ivi-pure-popt} of \num{50000} checked 
policies. The archive size for \emph{SPEA2} has been set to 
\num{50000}. Concerning the implementations, we have used OpenMP 
parallelization methods to use multiple CPU cores when possible. 
Furthermore, we have used advanced numerical algorithms to
evaluate~\eqref{eq:upper-lower-bound} and~\eqref{eq:average}. 
Specifically, for large instances, we substituted the direct LU 
solver~\cite{Stew94} by preconditioned GMRES~\cite{Saad93,Stew94} with an
ILU0-preconditioner. The code and testing infrastructure are available
at~\cite{Sc16bmdptool} and~\cite{Sc16bmdpinfrastructure}.

\subsection{The Model of Multi-Server Queue}
\label{sec:evaluation:model}

As the first case study, we have chosen a parameterizable model that can be 
easily generated; more concretely, we have considered a multi-server 
queue where servers can be switched off to save energy and switched on 
if the load in the system increases. Such queues are abstract models for 
server farms~\cite{GaHA10}. The goal is to find a compromise between 
small response times and low energy consumption.

We consider a system where customers arrive according to a Poisson 
process with rate $\lambda$ and require an exponentially distributed 
service with mean $\mu^{-1}$. As our algorithms are designed for 
discrete-time BMDPs, we consider a (morally) equivalent discrete model 
where the probability of arrival of a customer in a time unit is $p$, 
the service probability is $q$ and, thus $\lambda$ and $\mu$ are 
multiples of $p^{-1}$ resp. $q^{-1}$. The system has a capacity of $m$ 
and contains $c$ servers. Each server can be in one of three states 
{\em on}, {\em off\/} and {\em start}. A server can be switched off after 
the end of a service or if it is idle. A server that is switched off 
immediately changes its state from {\em on\/} to {\em off}. Servers in 
state {\em off\/} can be switched on which means that they change their 
state to {\em start}. The duration of the starting period is 
exponentially distributed with rate $\nu$, then the server changes its 
state to {\em on\/} and is ready to serve customers. A state of the system 
can be described by $(i,j,k,l)$ where $i \in [0,m]$ describes the number 
of customers, $j,k,l$ include the number of servers in state {\em on}, 
{\em start\/} and {\em off}, respectively. Consequently, $j+k+l=c$ has to 
hold. The number of states equals $n=(m+1)(c+2)(c+1)/2$. The reward in 
the state $(i,j,k,l)$ equals $(m-i)/(j\omega_1 + k\omega_2 + l\omega_3) 
$ where $\omega_1,\omega_2$, $\omega_3$ describe the energy consumption 
in the {\em on}, {\em start}, and {\em off\/} state.

The upper and lower bounds for the transition probabilities were chosen randomly
by generating Gaussian noise and adding it to the transition probabilities
defined above.
 
\paragraph{Comparison to a generic heuristic.}
\label{sec:evaluation:comparison-spea2}
As we consider a new problem, we have chosen to compare our approach to a
black-box heuristic. More specifically, we have considered multi-objective
optimization heuristics and decided to compare our algorithm to SPEA2 as it is a
well-studied evolutionary optimization algorithm that is specifically designed
to compute non-dominated sets for multi-objective optimization problems.

\textbf{The SPEA2 algorithm.}
In detail, SPEA2 keeps two sets: a \emph{population} $P$ and an 
\emph{archive} $A$ where the non-dominated solutions are stored. In each 
iteration of the optimization cycle, $A$ is updated with non-dominated 
elements of $P$. Then, a selection step takes place in which first, all 
elements of $A \cup P$ are assigned a \emph{fitness value} and then, 
the solutions with lower fitness values are chosen to generate new solutions by
application of mutation and crossover operators. The newly generated solutions
are then the new population.

The distinctive feature of this approach is the fitness evaluation: The 
fitness of an individual solution $p$ depends on the \emph{strength} of 
other solutions $p'$ that cover $p$. The strength itself is defined as 
the number of covered solutions; thus, the non-dominated solutions have 
maximal strength and minimal fitness values by definition, otherwise the 
ranking aims at picking more diverse solutions, i.\,e., solutions that 
are more evenly distributed in the objective value space.

In our SPEA2 implementation,%the multi-objective SBMDP optimization, 
we have defined problem-specific mutation and crossover operators. As 
possible solutions are stationary policies, the operators could be 
defined in a straightforward fashion. Mutation affects a decision in one 
state with probability $\nicefrac{1}{n}$, if $n$ is the number of states 
in the MDP, and replaces the previous action in the policy with a random 
one. Crossover takes two ``parent'' policies and chooses for the result 
an action from either of the original policies with probability 
$\nicefrac12$.

\textbf{Comparison metrics.} To quantify performance differences, we have used the coverage metric 
that has been introduced in~\cite{DBLP:journals/tec/ZitzlerT99}. This 
performance metric has been designed to compare two output sets of 
(heuristic) multi-objective optimization algorithms on the same 
problem and computes the fraction of one output that is covered 
(i.\,e., is dominated by or equal to) by an element of the other output 
set. Concretely, for two sets of vectors $X$ and $Y$, the coverage metric 
%$C(X, Y)$ 
is defined by
%\begin{equation}
$ C(X, Y) = \frac{1}{\envert{Y}}|{\cbr{y \in Y \mid \exists x \in X: x \ge 
y}}|$.
%\end{equation}
$C(X, Y) = 1$ means that all points in $Y$ are dominated by 
or equal to points in $X$ whereas $C(X, Y) = 0$ means that no 
point in $Y$ is covered by a point in $X$. It is worth noting that the 
coverage metric is asymmetric and in most cases it is interesting to 
compute both $C(X, Y)$ and $C(Y, X)$.

\textbf{Results.}
The results of the comparison can be found in Fig.~\ref{fig:c-ipt-evo}. The
first figure describes the coverage metric where the first argument is the
policy set computed by Alg.~\ref{alg:ivi-pure-popt}, the second figure describes
the converse. In the run, \emph{SPEA2} took always \num{1000} \si{s} while
Alg.~\ref{alg:ivi-pure-popt} did never take more than \num{330} \si{s}.

%%%\begin{figure*}[ht]
%%%%	\centering
%%%%	\begin{subfigure}[b]{.5\textwidth}
%%%%		\centering
%%%		\resizebox{.25\linewidth}{!}{\includegraphics[scale=1]{c_ipt_evo.png}}
%%%%		\includegraphics[width=.8\linewidth]{c_ipt_evo.png}
%%%%		\caption{$C_{IE}$}
%%%%	\end{subfigure}%
%%%%	\begin{subfigure}[b]{.5\textwidth}
%%%%		\centering
%%%		\resizebox{.25\linewidth}{!}{\includegraphics[scale=1]{c_evo_ipt.png}}
%%%%		\includegraphics[width=.8\linewidth]{c_evo_ipt.png}
%%%%		\caption{$C_{EI}$}
%%%%	\end{subfigure}
%%%	\caption{The $C$-measures in dependence of state space size: $C_{IE}$ (left) and $C_{EI}$ (right)}
%%%	\label{fig:c-ipt-evo}
%%%	  \vspace{-1.5em}
%%%\end{figure*}
%\begin{figure}[htbp]
%	\begin{minipage}[t]{0.5\linewidth}
%		\includegraphics[width=\linewidth]{c_ipt_evo.png}
%		\caption{$C_{IE}$}
%	\end{minipage}%
%	\hfill%
%	\begin{minipage}[t]{0.5\linewidth}
%		\includegraphics[width=\linewidth]{c_evo_ipt.png}
%		\caption{$C_{EI}$}
%	\end{minipage} 
%	\caption{The $C$-measures in dependence of state space size}
%	\label{fig:c-ipt-evo}
%\end{figure}

\begin{figure*}[ht]
  \begin{subfigure}{0.49\textwidth}
  	\resizebox{0.6\linewidth}{!}{
    \includegraphics[width=\textwidth]{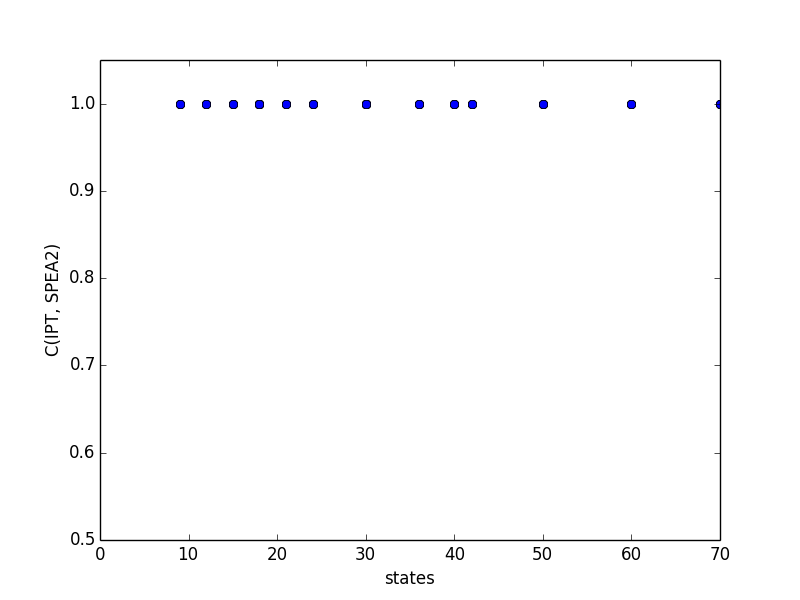}}
    \centering
    \caption{\footnotesize $C_{IE}$}
  \end{subfigure}
  \begin{subfigure}{0.49\textwidth}
  	\resizebox{0.6\linewidth}{!}{
    \includegraphics[width=\textwidth]{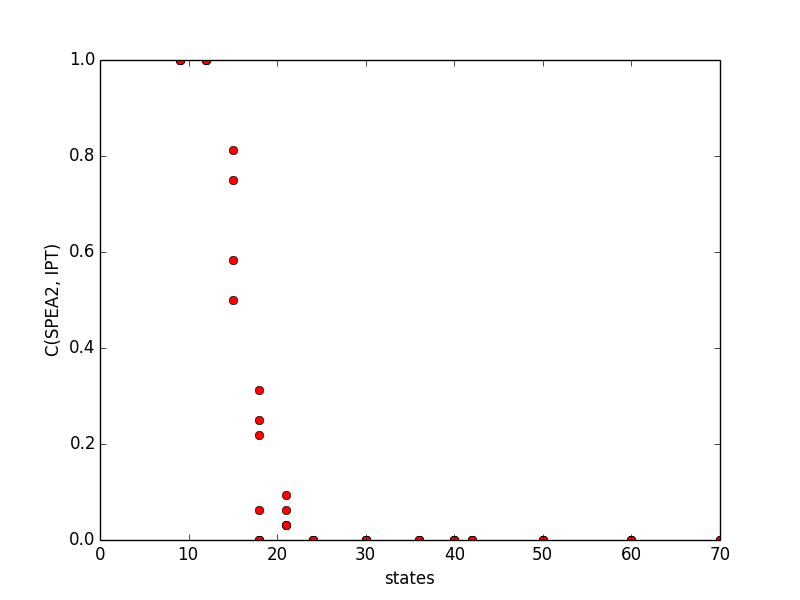}}
   \centering \caption{\footnotesize $C_{EI}$}
  \end{subfigure}
  \caption{\footnotesize The $C$-measures in dependence of state space size}
\label{fig:c-ipt-evo}
\end{figure*}
\ifthenelse{\boolean{long}}{
  The numeric results can be seen in Table~\ref{tbl:ipt-spea-c-metric}.
$i$ denotes the test case number, $T$ 
the time in seconds that Algorithm~\ref{alg:ivi-pure-popt} has used. 
The times for \emph{SPEA2} have not been shown as the algorithm has 
stopped on the timeout condition. $|P|$ describes the number of 
computed policies by each algorithm. $C_{IE}$ denotes the coverage 
metric where the first argument is the policy set computed by 
Alg.~\ref{alg:ivi-pure-popt} and the second one is the policy set 
computed by \emph{SPEA2}. $C_{EI}$ is the coverage metric with arguments 
swapped.
{\footnotesize

\begin{longtable}{l|l|l|l|l|l|l|l|l}
  \caption{Performance comparison of \emph{SPEA2} and 
Algorithm~\ref{alg:ivi-pure-popt}}
  \label{tbl:ipt-spea-c-metric} \\
  \hline
  \multicolumn{1}{c|}{$m$} &
  \multicolumn{1}{|c|}{$c$} & 
  \multicolumn{1}{|c|}{$|S|$} & 
  \multicolumn{1}{|c|}{$t$} & 
  \multicolumn{1}{|c|}{$T$ (Alg.~\ref{alg:ivi-pure-popt})} & 
  \multicolumn{1}{|c|}{$|P|$ (Alg.~\ref{alg:ivi-pure-popt})} &
  \multicolumn{1}{|c|}{$|P|$ (SPEA2)} & 
  \multicolumn{1}{|c|}{$C_{IE}$} & 
  \multicolumn{1}{|c}{$C_{EI}$} \\
  \hline
  \endfirsthead

  \multicolumn{1}{c|}{$m$} &
  \multicolumn{1}{|c|}{$c$} & 
  \multicolumn{1}{|c|}{$|S|$} & 
  \multicolumn{1}{|c|}{$t$} & 
  \multicolumn{1}{|c|}{$T$ (Alg.~\ref{alg:ivi-pure-popt})} & 
  \multicolumn{1}{|c|}{$|P|$ (Alg.~\ref{alg:ivi-pure-popt})} &
  \multicolumn{1}{|c|}{$|P|$ (SPEA2)} & 
  \multicolumn{1}{|c|}{$C_{IE}$} & 
  \multicolumn{1}{|c}{$C_{EI}$} \\
  \endhead
  \endfoot
  \endlastfoot
  $2$ & $1$ & $9$ & 
  1 & $0.1$ & 4 & 661 & $1.0$ & $1.0$ \\
  &&& 
  2 & $0.03$ & 4 & 831 & $1.0$ & $1.0$ \\
  &&& 
  3 & $<0.01$ & 2 & 557 & $1.0$ & $1.0$ \\
  &&& 
  4 & $<0.01$ & 8 & 1168 & $1.0$ & $1.0$ \\
  \hline
  $2$ & $2$ & $18$ & 
  1 & $0.04$ & 144 & 700 & $1.0$ & $0.0$
  \\
  &&&
  2 & $0.02$ & 90 & 273 & $1.0$ & $0.0$
  \\
  &&&
  3 & $0.01$ & 16 & 399 & $1.0$ & $0.0$
  \\
  &&&
  4 & $0.06$ & 192 & 235 & $1.0$ & $0.0$ 
  \\ \hline
  $2$ & $3$ & $30$ &
  1 & $3.16$ & 2048 & 1385 & $1.0$ & $0.0$
  \\ &&&
  2 & $8.83$ & 3456 & 1678 & $1.0$ & $0.0$
  \\ &&&
  3 & $27.9$ & 6480 & 2304 & $1.0$ & $0.0$
  \\ &&&
  4 & $32.09$ & 6912 & 2289 & $1.0$ & $0.0$
  \\ \hline
  $3$ & $1$ & $12$ & 
  1 & $0.01$ & 4 & 498 & $1.0$ & $1.0$
  \\
  &&&
  2 & $<0.01$ & 8 & 562 & $1.0$ & $1.0$
  \\
  &&&
  3 & $<0.01$ & 8 & 494 & $1.0$ & $1.0$
  \\
  &&&
  4 & $<0.01$ & 8 & 663 & $1.0$ & $1.0$
  \\ \hline
  $3$ & $2$ & $24$ & 
  1 & $30.01$ & 10368 & 1987 & $1.0$ & $0.0$
  \\ &&&
  2 & $1.21$ & 1728 & 1586 & $1.0$ & $0.0$
  \\ &&&
  3 & $0.95$ & 1536 & 1834 & $1.0$ & $0.0$
  \\ &&&
  4 & $0.57$ & 1024 & 1227 & $1.0$ & $0.0$
  \\ \hline
  $3$ & $3$ & 40 &
  1 & $134.25$ & 38626 & 3904 & $1.0$ & $0.0$
  \\ &&&
  2 & $170.0$ & 43446 & 1448 & $1.0$ & $0.0$
  \\ &&&
  3 & $163.02$ & 43555 & 3407 & $1.0$ & $0.0$
  \\ &&&
  4 & $158.31$ & 40326 & 2234 & $1.0$ & $0.0$
  \\ \hline
  $4$ & $1$ & 15 &
  1 & $0.01$ & 16 & 298 & $1.0$ & $0.8125$
  \\ &&&
  2 & $<0.01$ & 8 & 412 & $1.0$ & $0.75$
  \\ &&&
  3 & $<0.01$ & 4 & 839 & $1.0$ & $0.5$
  \\ &&&
  4 & $<0.01$ & 12 & 193 & $1.0$ & $0.583$
  \\ \hline
  $4$ & $2$ & 30 &
  1 & $124.77$ & 17281 & 1372 & $1.0$ & $0.0$
  \\ &&&
  2 & $15.94$ & 6144 & 1680 & $1.0$ & $0.0$
  \\ &&&
  3 & $0.86$ & 1152 & 2679 & $1.0$ & $0.0$
  \\ &&&
  4 & $0.18$ & 256 & 1145 & $1.0$ & $0.0$
  \\ \hline
  $4$ & $3$ & 50 &
  1 & $182.5$ & 44040 & 3110 & $1.0$ & $0.0$
  \\ &&&
  2 & $202.11$ & 43835 & 2873 & $1.0$ & $0.0$
  \\ &&&
  3 & $209.75$ & 39181 & 1345 & $0.9993$ & $0.0$
  \\ &&&
  4 & $196.69$ & 39348 & 2720 & $1.0$ & $0.0$
  \\ \hline
  $5$ & $1$ & 18 &
  1 & $0.01$ & 16 & 1466 & $1.0$ & $0.25$
  \\ &&&
  2 & $0.01$ & 16 & 1138 & $1.0$ & $0.3125$
  \\ &&&
  3 & $0.01$ & 32 & 1545 & $1.0$ & $0.21875$
  \\ &&&
  4 & $0.01$ & 16 & 1375 & $1.0$ & $0.0625$
  \\ \hline
  $5$ & $2$ & 36 &
  1 & $172.42$ & 44852 & 1931 & $1.0$ & $0.0$
  \\ &&&
  2 & $37.11$ & 8192 & 2310 & $1.0$ & $0.0$
  \\ &&&
  3 & $102.63$ & 13824 & 1867 & $1.0$ & $0.0$
  \\ &&&
  4 & $128.32$ & 38334 & 1970 & $0.9995$ & $0.0$
  \\ \hline
  $5$ & $3$ & 60 &
  1 & $195.88$ & 43520 & 3555 & $1.0$ & $0.0$
  \\ &&&
  2 & $233.39$ & 41321 & 3594 & $1.0$ & $0.0$
  \\ &&&
  3 & $203.18$ & 41957 & 2995 & $1.0$ & $0.0$
  \\ &&&
  4 & $186.82$ & 36129 & 2493 & $1.0$ & $0.0$
  \\ \hline
  $6$ & $1$ & 21 &
  1 & $0.02$ & 64 & 605 & $1.0$ & $0.09375$
  \\ &&&
  2 & $0.01$ & 32 & 977 & $1.0$ & $0.03125$
  \\ &&&
  3 & $0.01$ & 64 & 1130 & $1.0$ & $0.03125$
  \\ &&&
  4 & $0.01$ & 64 & 756 & $1.0$ & $0.0625$
  \\ \hline
  $6$ & $2$ & 42 &
  1 & $193.48$ & 43583 & 2307 & $0.9996$ & $0.0$
  \\ &&&
  2 & $197.19$ & 46716 & 2687 & $1.0$ & $0.0$
  \\ &&&
  3 & $148.53$ & 41420 & 1631 & $1.0$ & $0.0$
  \\ &&&
  4 & $254.92$ & 18432 & 2657 & $1.0$ & $0.0$
  \\ \hline
  $6$ & $3$ & 70 &
  1 & $355.77$ & 40955 & 2563 & $1.0$ & $0.0$
  \\ &&&
  2 & $321.47$ & 41570 & 2506 & $1.0$ & $0.0$
  \\ &&&
  3 & $330.95$ & 41739 & 3310 & $1.0$ & $0.0$
  \\ &&&
  4 & $330.95$ & 36618 & 3288 & $1.0$ & $0.0$
\end{longtable}}
}{}

It is easy to see that Algorithm~\ref{alg:ivi-pure-popt} almost 
always delivers a significantly better performance with respect to both 
time complexity as well as quality of computed policies. More 
concretely, in nearly all cases Algorithm~\ref{alg:ivi-pure-popt} has 
computed a set that completely covered all solutions generated by 
\emph{SPEA2}; the evolutionary heuristic, however, has never produced a 
policy that strictly dominated a policy from 
Alg.~\ref{alg:ivi-pure-popt} and was only able to yield comparable 
solutions on small instances with state space size of at most \num{20}.

\paragraph{Comparison to an exact computation.}
\label{sec:evaluation:comparison-exact}
For some instances, we have furthermore compared the performance of
Algorithm~\ref{alg:ivi-pure-popt} to the exact approach in
Algorithm~\ref{alg:ppareto-bfs}. Concretely, we have considered the case $m = 2,
c = 3$. It has turned out that for this case, the coverage metric was always
$1$. This suggests that Algorithm~\ref{alg:ivi-pure-popt} may compute the
complete Pareto frontier not only heuristically but also in theory. This is,
however, a conjecture subject to further investigation.

%\textbf{Time complexity.}
\subsection{Time complexity measurements}
\label{sec:evaluation:complexity}
As the number of policies was bounded by an upper limit of \num{50000} and $|A|
= o(|S|)$, the complexity can be roughly estimated by a cubic term in $|S|$. For
practical applications, we are also interested in runtimes on real-life
instances. For this, we have used a slightly different but more general (and
somewhat more scalable) model.
\vh{Shall we highlight the following as another case study, e.g., ``The Grid Model"?}
\ds{Does it look better now?}

\textbf{The grid model.}
We have chosen a model that resembles 
a grid with $n m$ states $S = \cbr{s_{i, j} \mid i \in [n], j \in 
[m]}$ and $m$ actions $A = [m]$. The rewards for actions 
in each state are chosen randomly with mean  \num{100} and variance \num{20}. 
The transition probabilities are also chosen randomly according to the Dirichlet
distribution. Concretely, the 
transition probability vector from state $s_{i, j}$ to states $s_{\min(n, i+1), 
j'}$ for action $a$ is Dirichlet-distributed with concentration 
parameters $\del{\alpha_1^{a}, \ldots, \alpha_m^{a}}$ where
$\alpha_{j'}^{a} = 10$ if $a = j'$ and $\alpha_{j'}^{a} = 1$ otherwise
which yields an (expected) 10 times larger probability to land in $s_{i+1,a}$
than in other states. The upper and lower bounds are, as before,
generated by adding Gaussian noise.

\textbf{Complexity.}
We present the results in graphical form. To get the results, we have 
run Algorithm~\ref{alg:ivi-pure-popt} on instances with up to \num{400} 
states, with $n$ and $m$ between 5 and 20. For each $(n, m)$, we 
have created \num{4} instances.

The results can be seen in Fig.~\ref{fig:pareto-total-time-plot} and 
Fig.~\ref{fig:pareto-time-per-policy-plot}. The red dots are the 
empirical data, the blue lines describe a confidence interval that 
stems from a (scaled) t-distribution guess. The green line is the 
cubic regression term for convenience. One can see that the experimental
performance generally matches a polynomial term, but also that even on large
instances, the mean time until a non-dominated solution is generated lies under
a second.
\begin{figure*}[ht]
	\centering
	\begin{subfigure}[b]{.5\textwidth}
		\centering
		\resizebox{0.6\linewidth}{!}{\includegraphics{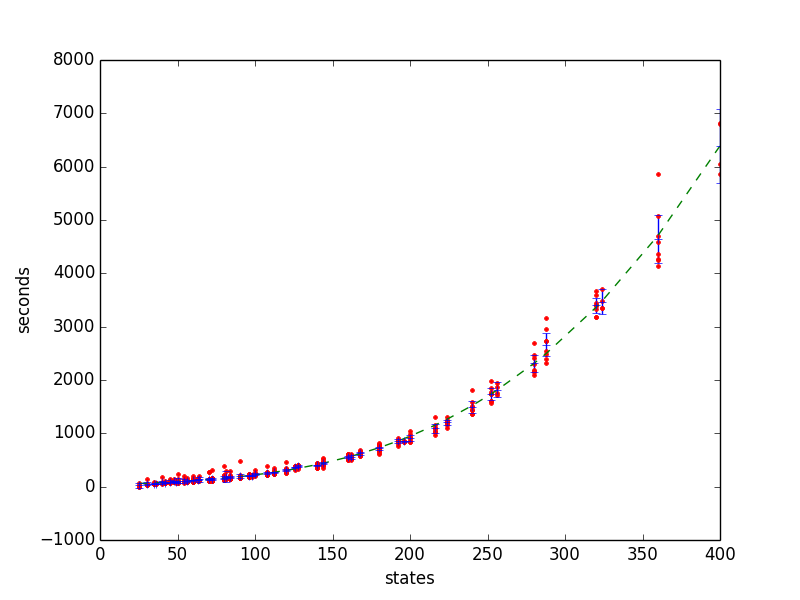}}
		\caption{\footnotesize Total time in dependence of problem size}
		\label{fig:pareto-total-time-plot}
	\end{subfigure}%
	  \hfill
	\begin{subfigure}[b]{0.5\textwidth}
		\centering
		\resizebox{0.6\linewidth}{!}{\includegraphics{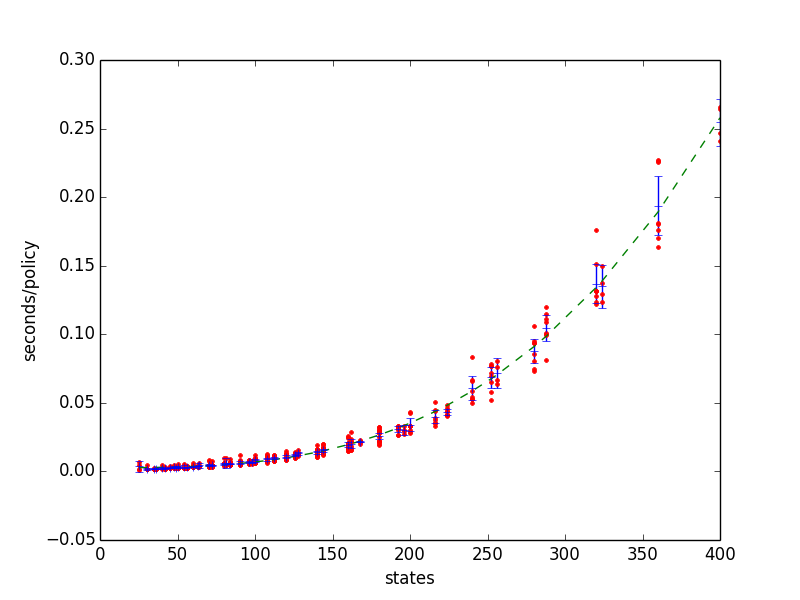}}
		\caption{\footnotesize Mean time for a policy in dependence of problem size}
		\label{fig:pareto-time-per-policy-plot}
	\end{subfigure}
	\caption{\footnotesize Time complexity measurements on the grid model}
	\label{fig:grid-model}
%	  \vspace{-1.5em}
\end{figure*}
%\begin{figure}[htbp]
%	\begin{minipage}[t]{0.5\linewidth}
%		\includegraphics[width=\linewidth]{total_time_plot.png}
%		\caption{\footnotesize{Total time in dependence of problem size}}
%		\label{fig:pareto-total-time-plot}
%	\end{minipage}%
%	\hfill%
%	\begin{minipage}[t]{0.5\linewidth}
%		\includegraphics[width=\linewidth]{time_per_policy_plot.png}
%		\caption{\footnotesize{Mean time for a policy in dependence of problem size}}
%			\label{fig:pareto-time-per-policy-plot}
%	\end{minipage} 
%	\caption{Time complexity measurements on the grid model}
%	\label{fig:grid-model}
%\end{figure}
%\begin{figure}
%  \centering
%  \includegraphics[width=0.8\linewidth]{total_time_plot.png}
%  \caption{Total time in dependence of problem size}
%\label{fig:pareto-total-time-plot}
%\end{figure}
%
%\begin{figure}
%  \centering
%  \includegraphics[width=0.8\linewidth]{time_per_policy_plot.png}
%  \caption{Mean time for a policy in dependence of problem size}
%\label{fig:pareto-time-per-policy-plot}
%\end{figure}

\subsection{The Model of Autonomous Nondeterministic Tour
	Guides}
Our second case study is inspired by ``Autonomous Nondeterministic Tour
Guides'' (ANTG) in~\cite{CRI07,hashemi2016reward}, which models a complex
museum with a variety of collections. 

Models in~\cite{CRI07} are MDPs. In our experiment, we will insert some
uncertainties into the MDP. Due to the popularity of the
museum, there are many visitors at the same time. Different visitors
may have different preferences of arts. We assume the museum divides
all collections into different categories so that visitors can choose
what they would like to visit and pay tickets according to their
preferences. In order to obtain the best experience, a visitor can first
assign certain weights to all categories denoting their preferences to
the museum, and then design the best strategy for a target. However,
the preference of a sort of arts to a visitor may depend on many
factors like price, weather, or the length of queue at that moment
etc., hence it is hard to assign fixed values to these preferences. In
our model we allow uncertainties of preferences such that their values
may lie in an interval. 

For simplicity we assume all collections are organized in an $n\times
n$ square with $n\ge 10$. Let $m=\frac{n-1}{2}$. We assume all collections
at $(i,j)$ are assigned with a weight 1 if $\abs{i-m}>\frac{n}{5}$ or
$\abs{j-m}>\frac{n}{5}$, with a weight 2 if
$\abs{i-m}\in(\frac{n}{10},\frac{n}{5}]$ or $\abs{j-m}\in(\frac{n}{10},\frac{n}{5}]$; otherwise they are assigned with a weight interval
$[3,4]$. In other words, we expect 
collections in the middle will be more popular and subject to more
uncertainties than others. Furthermore, we assume that people at each location $(i,j)$
have two non-deterministic choices: either move to the north and east, that is, $(\min(n, i+1), \min(n,
j+1))$
or to the north and west, that is, $(\min(n, i+1), \max(0, j-1))$
if $i\ge j$, while if $i\le j$, they can move
either to the north and east, 
$(\min(n, i+1), \min(n, j+1))$% the south and west
, or to the south and east, that is, $(\max(0, i-1), \min(n, j+1))$%
% the south and east
. The transitions also depend on the location of the collection. For the
collections in the middle, the main direction of transition is chosen with
probability $[0.8, 1]$ while the probability to move to some other neighbor
collection is $[0, 0.2]$. In the expected case, we set the probability to move
to the collection in the main direction to $0.8$ and distribute the remaining
probability mass evenly among other neighbor collections. For collections
outside the middle, the main direction (for example, north and west) is chosen
with probability $1$.

Therefore a model with parameter $n$ has $n^2$ states in total and
roughly $2n^2$ transitions, 2\% of which are associated with uncertain weights
and uncertain transition probabilities. Notice that a transition with uncertain
weights essentially corresponds to several transitions with concrete weights. 
%In
%each ANTG model, the 2\% transitions with uncertain transitions contribute to
%about 20\% of transitions in the resultant MDP.

We define a reward structure denoting the reward one can obtain by visiting each
collection. For simplicity, we let the reward be the same as
the weight of a collection. 
%Let the point $(0,0)$ be the entrance and
%$(n-1,n-1)$ the exit. 
We can ask for the optimal policy for the expected discounted reward
criterion, that is, in the scenario where it is preferable to make better
rewarding moves early.

%We can ask questions like ``Whether it is
%possible to go through the museum, i.e., from the entrance to the exit, with probability greater than 0.9,
%while the accumulated reward is not
%greater than $\mathcal{R}$, i.e.,
%$\mathbf{Pr}_{\ge 0.9}(\mathbf{F}^{\le \mathcal{R}}\mathit{exit})$''. 

\paragraph{Evaluation of the ANTG model}

We have performed an evaluation of the model for $10 \le n \le 20$, with the
results depicted in Fig.~\ref{fig:tour-guide-evaluation}. For convenience, the
runtime and the number of policies are plotted in dependence of the number of
states, which is $n^2$. We see that on large instances, the problem structure
yields a large number of optimal policies, thus decreasing performance. On small
instances, however, the number of optimal policies generated is small, which
allows for a fast computation of the Pareto frontier.

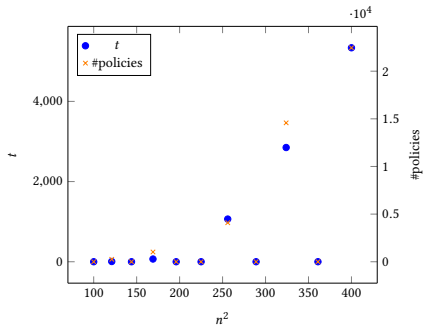
\begin{figure}[ht]
  \centering
  \scalebox{1}{
  \begin{tikzpicture}[scale=0.6]
    \pgfplotsset{
      % scale only axis,
      y axis style/.style={
        yticklabel style=#1,
        ylabel style=#1,
        y axis line style=#1,
        ytick style=#1
      }
    }
    \begin{axis}[xlabel={$n^2$}, ylabel={$t$}, axis y line*=left,
      legend pos=north west]
      \addplot[mark=*,blue, only marks] coordinates {
        (100, 0.09)
        (121, 6.03)
        (144, 0.07)
        (169, 68.01)
        (196, 0.12)
        (225, 0.14)
        (256, 1064.63)
        (289, 0.41)
        (324, 2849.56)
        (361, 0.49)
        (400, 5337.49)
      }; \label{plot-time}
    \end{axis}
    \begin{axis}[xlabel={$n^2$}, ylabel={\#policies}, axis y line*=right, 
        ylabel near ticks,
        axis x line=none, legend pos=north west]
      \addlegendimage{/pgfplots/refstyle=plot-time}\addlegendentry{$t$}
      \addplot[mark=x, orange, only marks] coordinates {
        (100, 4)
        (121, 257)
        (144, 4)
        (169, 1025)
        (196, 3)
        (225, 1)
        (256, 4096)
        (289, 1)
        (324, 14579)
        (361, 1)
        (400, 22457)
      }; \addlegendentry{\#policies} \label{plot-n}
    \end{axis}
  \end{tikzpicture}
  }
  \caption{\footnotesize Evaluation of the ANTG model}
  \label{fig:tour-guide-evaluation}
    %\vspace{-1em}
\end{figure}

\section{Conclusions}
\label{sec:outro}
\ifthenelse{\boolean{long}}{
Summing up, we have presented a novel approach to analyze 
bounded-parameter Markov decision processes. In contrast to known 
approaches~\cite{DBLP:journals/ai/GivanLD00} that usually analyze the 
worst case behavior and result in a variant of robust optimization, the 
problem is handled here as a multi-objective problem. Of course, the 
price for this extension can be an exponential complexity due to an 
exponential number of {\em optimal\/} and mutually incomparable policies 
in this case. The problem differs from many other multi-objective 
optimization approaches for MDPs where several expected values are 
analyzed, whereas here worst, expected and best case behavior are 
considered together which together yields a \emph{multi-scenario\/} 
optimization problem. 
%This has for example the consequence that 
%randomized policies resulting from the combination of two or more 
%deterministic policies are not applicable since the worst case may 
%correspond to the worst case of the set of policies that are used in the 
%approach. 

For the problem of finding all Pareto optimal policies in 
multi-objective optimization of BMDPs, two computational algorithms are 
presented. One of the algorithms computes the desired set of policies 
exactly, but has a prohibitive runtime complexity. The second algorithm 
that is explicitly designed for the problem is a heuristic. Although it 
also cannot avoid exponential complexity, if the result size is 
exponential, it seems to perform well on a large number of test 
instances. We have compated our heuristic to a generic 
evolutionary multi-objective optimization algorithm with promising 
results; the policies obtained by the heuristic were in almost all 
cases better than those produced by the evolutionary algorithm.

The approach presented here can be extended in various directions. As 
mentioned it is fairly easy to consider additional goals beyond worst, 
average and best cases. Our results not only apply to bounded-parameter 
Markov decision processes, but can also be utilized for Markov decision 
processes with convex uncertainties%
%, as the convex program $\min 
%\vec c^\top \vec x$ subject to $\vec x \in C$ can 
%be solved efficiently
~\cite{DBLP:conf/cav/PuggelliLSS13}. Then, 
the same basic algorithms can be used in order to compute Pareto optimal 
policies, if one adjusts them to solve the convex program in the 
iteration step.

For future research, it would be interesting to consider the 
scalarization problem where the goal is to optimize the weighted 
sum of the value vectors; as this problem seems to be non-local and 
non-unimodal, it may need new methods for problem analysis. 
%It would 
%also be interesting to extend the methods to unrestricted stochastic 
%games~\cite{Shap53} and other formalisms like partially 
%observable MDPs~\cite{Lo91}. 
Also, the presented results may be reused 
in order to develop algorithms for other optimality criteria such as 
expected gain.
}
{
In this paper, we have presented a novel approach to analyze 
bounded-parameter Markov decision processes. In contrast to known 
approaches~\cite{DBLP:journals/ai/GivanLD00} that usually analyze the 
worst case behavior and result in a variant of robust optimization, the 
problem is handled here as a multi-objective problem. Of course, the 
price for this extension can be an exponential complexity due to an 
exponential number of {\em optimal\/} and mutually incomparable policies 
in this case. The problem differs from many other multi-objective 
optimization approaches for MDPs where several expected values are 
analyzed, whereas here worst, expected and best case behavior are 
considered together which together yields a \emph{multi-scenario\/} 
optimization problem. 
%This has for example the consequence that 
%randomized policies resulting from the combination of two or more 
%deterministic policies are not applicable since the worst case may 
%correspond to the worst case of the set of policies that are used in the 
%approach. 

In order to find all Pareto optimal policies in multi-objective optimization of BMDPs, 
two computational algorithms are presented. One of the algorithms computes the desired set of policies 
exactly, but has a prohibitive runtime complexity. The second algorithm 
that is explicitly designed for the problem is a heuristic and seems to perform well on a 
large number of test instances. In particular, we have shown that the policies obtained by the heuristic 
were in almost all cases better than those produced by a generic 
evolutionary multi-objective optimization algorithm.  
%
%Although it 
%also cannot avoid exponential complexity, if the result size is 
%exponential, it seems to perform well on a large number of test 
%instances. 
%We have compated our heuristic to a generic 
%evolutionary multi-objective optimization algorithm with promising 
%results; the policies obtained by the heuristic were in almost all 
%cases better than those produced by the evolutionary algorithm.
%
%For the problem of finding all Pareto optimal policies in 
%multi-objective optimization of BMDPs, two computational algorithms are 
%presented. One of the algorithms computes the desired set of policies 
%exactly, but has a prohibitive runtime complexity. The second algorithm 
%that is explicitly designed for the problem is a heuristic. Although it 
%also cannot avoid exponential complexity, if the result size is 
%exponential, it seems to perform well on a large number of test 
%instances. We have compated our heuristic to a generic 
%evolutionary multi-objective optimization algorithm with promising 
%results; the policies obtained by the heuristic were in almost all 
%cases better than those produced by the evolutionary algorithm.

The approach presented here can be extended in various directions. As 
mentioned it is fairly easy to consider additional goals beyond worst, 
average and best cases. Our results not only apply to bounded-parameter 
Markov decision processes, but can also be utilized for Markov decision 
processes with convex uncertainties% as the convex program $\min 
%\vec c^\top \vec x$ subject to $\vec x \in C$ can 
%be solved efficiently
~\cite{DBLP:conf/cav/PuggelliLSS13}. Therefore, 
the same basic algorithms can be used in order to compute Pareto optimal 
policies, if one adjusts them to solve the convex program in the 
iteration step.

For future research, it would be interesting to consider the 
stochastic multi-scenario problem for stochastic BMDPs. There, the goal is to
optimize the expected value of the value vector, in contrast to the value vector
of the expected MDP we have used here. This problem is a slightly different,
single-objective optimization problem and needs to be considered separately.
%It would 
%also be interesting to extend the methods to unrestricted stochastic 
%games~\cite{Shap53} and other formalisms like partially 
%observable MDPs~\cite{Lo91}. 
%%Also, the presented results may be reused 
%%in order to develop algorithms for other optimality criteria such as 
%%expected gain.	
}
\appendix

\ifthenelse{\boolean{complexity}}{%
\section{Proof of Theorem~\ref{thm:pareto-nphard}}
\label{sec:proof-hardness}

  The stated problem is obviously in \NP{} because the evaluation of a
policy for a SBMDP can be done in polynomial time. As for \NP-hardness,
we show a reduction from the subset sum problem. Given a subset sum
instance $M = \cbr{m_1, \ldots, m_n}$ we construct a SBMDP and two
vectors $\vec v_1, \vec v_2$ such that there is a policy $\pi$ with
$\vec v_{\downarrow}^{(\pi)} \ge \vec v_1, \overline{\vec v}^{(\pi)} \ge
\vec v_2$ if and only if there is a subset $I \subseteq [n]$ such that
$\sum_{i \in I} m_i = \frac{1}{2} \sum_{i = 1}^n m_i$.

  The SBMDP which we shall construct will have an arbitrary nonzero
discount factor $\gamma \in \intoo{0,1}$, $4n + 1$ states which have the
identifiers $t$ and $q_i, \overline{s}_i, s_{\downarrow i}, s_{\uparrow
i}$ for $i \in [n]$.
The general idea of the reduction is the following: The SBMDP shall
proceed through all states $q_i$ and end up in the absorbing state $t$.
In the state $q_i$, it shall be possible to choose between the (adjusted
for the discount factor) reward pairs $\del{2m_i,0}$ and $\del{m_i,
m_i}$ with the first component being the average case and the second
component being the worst case part. This way, a pure policy $\pi$ will
induce a subset $I \subseteq [n]$ such that the total reward, if one
starts in state $q_1$, will be $\del{\sum_{i=1}^n m_i + \sum_{i \in I}
m_i, \sum_{i \not\in I} m_i}$. Technically, we model this by enabling
two actions in states $q_i$, $a$ and $b$. These actions do not generate
rewards, but lead to different outcomes: The action $a$ leads to the
state $\overline s_i$ unconditionally with reward $0$, and all actions
from $\overline s_i$ lead to the state $q_{i+1}$, if $i < n$, or $t$,
if $i = n$, with reward $\frac{m_i}{\gamma^{2i-1}}$. The action $b$
leads to either $s_{\downarrow i}$ or $s_{\uparrow i}$, with probability
in the interval $\intcc{0,1}$; all actions from these two states lead,
in turn, to $q_{i+1}$ (or $t$, if $i = n$), and the difference between
$s_{\downarrow i}$ and $s_{\uparrow i}$ lies in the reward: the reward
in $s_{\downarrow i}$ is $0$, the reward in $s_{\uparrow i}$ is
$\frac{4m_i}{\gamma^{2i-1}}$.

  Finally, we define the expected discounted rewards we would like to
get with a pure policy $\pi$, it should be $\frac12 \sum_{i \in [n]}
m_i$ in the worst case and $\frac32 \sum_{i \in [n] } m_i$ in the
average case.

  \begin{figure}[htp]
    \centering
    \begin{tikzpicture}[auto,>=stealth',node distance=2cm]
      \node[circle,draw] (qi) at (0,0) {$q_i$};
      \coordinate (qui) at (0,1.1);
      \node[circle,draw] (sui) at (1.1,2.2) {$s_{\uparrow i}$};
      \node[circle,draw] (sli) at (-1.1,2.2) {$s_{\downarrow i}$};
      \coordinate (qli) at (0,-1.1);
      \node[circle,draw] (sai) at (1.1, -2.2) {$\overline s_i$};
      \node[circle,draw] (qi1) at (3.3,0) {$q_{i+1}$};

      \path[-] (qi) edge [] node {$b, 0$} (qui);
      \path[-] (qi) edge [] node {$a, 0$} (qli);
      
      \path[->] (qli) edge [bend right] node {$p = 1$} (sai);
      \path[->] (sai) edge [bend right] node [swap] {$\cdot,
\frac{m_i}{\gamma^{2i-1}}, p = 1$} (qi1);

      \path[->] (qui) edge [bend left] node [swap] {$p \in \intcc{0,1}$}
(sui);
      \path[->] (qui) edge [bend right] node {$p \in \intcc{0,1}$}
(sli);
      \path[->] (sui) edge [bend left=60] node {$\cdot,
\frac{4m_i}{\gamma^{2i-1}}, p = 1$} (qi1);
      \path[->] (sli) edge [bend left=60] node {$\cdot, 0, p = 1$}
(qi1);
    \end{tikzpicture}
    \caption{Construction from Theorem~\ref{thm:pareto-nphard}}
  \end{figure}

It is easy to see that the construction can be done in polynomial time.
We want now to show correctness. For every subset $I \subseteq [n]$ we
define a policy $\pi_I$ with $\pi_I(q_i) = a$ if $i \in I$ and
$\pi_I(q_i) = b$ if $i \not \in I$. The expected discounted total reward
for policy $\pi_I$ in state $q_1$ will then be $\sum_{i \in I}
\gamma^{2i-1} \cdot \frac{m_i}{\gamma^{2i-1}} + \sum_{i \in [n]
\setminus I} \gamma^{2i-1} \cdot \frac{2m_i}{\gamma^{2i-1}} = \sum_{i
\in [n]} m_i + \sum_{i \in [n] \setminus I} m_i$ in the average case
and $\sum_{i \in I} \gamma^{2i-1} \cdot \frac{m_i}{\gamma^{2i-1}} =
\sum_{i \in I} m_i$ in the worst case. If there is a subset $I \subseteq
[n]$ such that $\sum_{i \in I} m_i = \sum_{i \in [n] \setminus I} m_i$,
then there is a policy $\pi_I$ that yields the requested expected
discounted rewards. Furthermore, as any pure policy $\pi$ induces a
subset $I \subseteq [n]$ by defining $i \in I \Leftrightarrow \pi_{q_i}
= a$, the existence of a pure policy $\pi$ with the requested expected
discounted rewards implies the existence of a subset with sum $\frac12
\sum_{i \in [n]} m_i$.

It is easy to see that the proof technique can also be applied to the
combination of best-case and expected-case measures, and best-case and
worst-case measures.
}{}

\ifthenelse{\boolean{nonstationary}}{%
\section{Proof of Theorem~\ref{thm:pareto-error-bounds}}
\label{sec:proof-error-bounds}

Since the policies computed in Algorithm~\ref{alg:ivi-pareto-popt} are
deterministic, they can be represented by vectors $\pi_i$ for the
decisions in the $i$-th ($1 \le i \le k$) step. Each
$(\mat{v}_{\downarrow}, \overline{\mat{v}}, \mat{v}_{\uparrow}) \in V^k$
results from a specific sequence $\pi_1,\ldots,\pi_k$.
$(\mat{v}_{\downarrow}, \overline{\mat{v}}, \mat{v}_{\uparrow})$ has
been computed starting with value vector $(\mat{0},\mat{0},\mat{0})$
$k$ steps apart. Now assume that we start with $(\mat{w}_{\downarrow},
\overline{\mat{w}}, \mat{w}_{\uparrow})$ instead of
$(\mat{0},\mat{0},\mat{0})$ and use the same policy
$\pi_1,\ldots,\pi_k$, then the resulting value vectors are
\[
\begin{array} {lll}
\mat{v}_{\downarrow} + \gamma^{k} \prod\limits_{i=1}^k  
  P^{(\pi_i)}\mat{w}_{\downarrow} & \le \mat{v}_{\downarrow} +
\gamma^{k} \max_{s \in S} \left(\mat{w}_{s \downarrow}\right)\idvt ,\\
\overline{\mat{v}} + \gamma^{k} \prod\limits_{i=1}^k P^{(\pi_i)}
  \overline{\mat{w}} & \le \overline{\mat{v}} + \gamma^{k} \max_{s \in
S}
  \left(\overline{\mat{w}}_s\right)\idvt & \mbox{ and }\\
\mat{v}_{\uparrow} + \gamma^{k} \prod\limits_{i=1}^k  
  P^{(\pi_i)}\mat{w}_{\uparrow} & \le \mat{v}_{\uparrow} +
\gamma^{k} \max_{s \in S} \left(\mat{w}_{s \uparrow}\right)\idvt.
\end{array}
\]
Consequently, $(\mat{v}_{\uparrow}, \overline{\mat{v}},
\mat{v}_{\downarrow})$ contains the discounted reward accumulated in the
last $k$ steps and is a lower bound for the discounted reward
accumulated over an infinite number of steps. Since $V^k$ contains all
Pareto optimal vectors reachable after $k$ steps, $v^*_{\downarrow}$ is
the maximal lower bound of the accumulated reward in $k$ steps,
$\overline{v}^*$ is the maximal average reward that can be accumulated
in $k$ steps and $v^*_{\uparrow}$ is the maximal upper bound of the
accumulated reward in $k$ steps. This implies
\[\]
\begin{array}{llll}
\max_{s \in S} \left(\mat{w}_{s \downarrow} \right) & \le
\sum\limits_{i=0}^\infty \left(\gamma^{k} \right)^i v_{\downarrow}^* &
= \frac{1}{1-\gamma^{k}}v_{\downarrow}^* , \\
\max_{s \in S} \left(\overline{\mat{w}}_s \right) & \le
\sum\limits_{i=0}^\infty \left(\gamma^{k} \right)^i \overline{v}^* & =
\frac{1}{1-\gamma^{k}}\overline{v}^* & \mbox{ and }\\
\max_{s \in S} \left(\mat{w}_{s \uparrow} \right) & \le
\sum\limits_{i=0}^\infty \left(\gamma^{k} \right)^i v_{\uparrow}^* &
= \frac{1}{1-\gamma^{k}}v_{\uparrow}^*
\end{array}
\]
for every value vector $(\mat{w}_{\downarrow}, \overline{\mat{w}},
\mat{w}_{\uparrow})$ reachable by an arbitrary policy and also for every
value vector $(\mat{w}_{\downarrow}, \overline{\mat{w}},
\mat{w}_{\uparrow}) \in \mathcal{V}_{opt}$. 

Let $(\mat{u}_{\downarrow}, \overline{\mat{u}}, \mat{u}_{\uparrow}) \in
\mathcal{V}_{opt}$ be a Pareto optimal pair of vectors resulting from
policy $f = \pi_1,\ldots, \pi_k, \ldots \in \mathcal{P}_{opt}$. First,
assume that $\pi_1,\ldots,\pi_k$ is a policy that has been used in
Algorithm~\ref{alg:ivi-pareto-popt} to compute $(\mat{v}_{\downarrow},
\overline{\mat{v}}, \mat{v}_{\uparrow}) \in V^k$. Then
$\overline{\mat{v}}$ contains the average reward accumulated during the
last $k$ steps by policy $f$. Additionally, $\overline{\mat{u}}$
contains the discounted reward accumulated in the steps more than $k$
steps apart. A bound for this reward is given above and this bound is
weighted by $\gamma^{k}$ which results in the bound given in the
theorem. With the same arguments the bound for the worst and best case
behavior can be derived.

Now assume that $f = \pi_1,\ldots, \pi_k; \ldots \in \mathcal{P}_{opt}$
but $\pi_1,\ldots,\pi_k$ has not been used to compute a vector from
$V^k$. Then $(\mat{u}_{\downarrow}, \overline{\mat{u}},
\mat{u}_{\uparrow})$ consists of a part $(\mat{w}_{\downarrow},
\overline{\mat{w}}, \mat{w}_{\uparrow})$ that includes the reward
accumulated during the last $k$ steps and the rest which can be bounded
as in the previous case. However, since $V^k$ contains all Pareto
optimal vectors with respect to the reward accumulated in $k$ steps, it
contains a vector pair $(\mat{v}_{\downarrow}, \overline{\mat{v}},
\mat{v}_{\uparrow})$ with $\mat{v}_{\downarrow} \ge
\mat{w}_{\downarrow}$, $\overline{\mat{v}} \ge \overline{\mat{w}}$ and
$\mat{v}_{\uparrow} \ge \mat{w}_{\uparrow}$. These vectors plus the
bound for the reward accumulated in the steps $k+1, \ldots$ define a
bound for $(\mat{u}_{\downarrow}, \overline{\mat{u}},
\mat{u}_{\uparrow})$. 
}{}

%\input{related-work}
%\section{Multi-objective Robust Strategy Synthesis for IMDP\lowercase{s}}
%\label{sec:multiObRobustSyn}
%\input{multi-objective-queries}
%\input{strategy-synthesis}
%
%%\clearpage
\ifthenelse{\boolean{acknowledgments}}{\section*{Acknowledgments}
This work is supported by the DFG as part of RTG 1855 and by the ERC Advanced Grant 695614 (POWVER).}{}
%by the CAS/SAFEA International Partnership Program for Creative Research Teams, 
%by the National Natural Science Foundation of China (Grants No.\ 61472473, 61532019, 61550110249, 61550110506),
%by the Chinese Academy of Sciences Fellowship for International Young Scientists, 
%and 
%by the CDZ project CAP (GZ 1023).
%
%%
%%\bigskip
%%\noindent Thank you for reading these instructions carefully. We look forward to receiving your electronic files!
%
%\clearpage
%\pagebreak
\bibliographystyle{ACM-Reference-Format}
\bibliography{biblio}
%\clearpage
\ifthenelse{\boolean{final}}{}{%
  \section*{Appendix} 
\appendix 

This appendix contains the proofs of the results enunciated in the main part of the paper.
%It is available for the reviewers in case they want to verify the correctness of the presented results;
%it is not meant to be included in the final version of the paper.

%=============================================
\section{Proofs of the Results Enunciated in the Paper} % in Section~\ref{Sec:compositionality}}
\label{app:proofs-of-theorems}
%=============================================
\lemmaNondominatedPath*
\begin{proof}
  We provide a proof by induction on $d(\pi, \pi')$. For $d(\pi,
  \pi') \in \cbr{0, 1}$, the statement holds obviously. 

  For $d(\pi, \pi') = c > 1$, the induction hypothesis is that 
  the statement holds for $c - 1$. This means that for each policy 
  $\pi_1$ with distance $d(\pi_1, \pi') = c-1$ there exists a sequence of 
  policies $\pi_1, \pi_2, \ldots, \pi_{c} = \pi'$ such that for any two 
  adjacent policies $\pi_i, \pi_{i+1}$ it is $\vec{v}^{(\pi_i)} \not> 
  \vec{v}^{(\pi_{i+1})}$.

  To show the induction step, we must infer the statement for $d(\pi, 
  \pi') = c$. Suppose now for the sake of contradiction that it is not 
  the case. 
  % We shall construct a policy $\pi_1$ with 
  % $d(\pi_1, \pi') = c - 1$ and $d(\pi, \pi_1) = 1$ that violates this 
  % assumption. 
  We observe that under this assumption, for each state $s \in S$, the policy
  $\pi^{(s, \pi'(s))}$ that results from changing $\pi$ in state $s$ to choose
  action $\pi'(s)$ results in a value vector that is dominated by
  $\vec{v}^{(\pi)}$, i.\,e., $\vec{v}^{(\pi)} > \vec{v}^{(\pi^{(s, \pi'(s))})}$.
  Let us now consider a restricted SBMDP $\mathcal{P}^{[\pi, \pi']} = (S,
  A^{[\pi, \pi']}, \bp{T}^{[\pi, \pi']}, \bp{R}^{[\pi, \pi']})$ where the
  available actions are only those used in either $\pi$ or $\pi'$, that is,
  $A^{[\pi, \pi']} = \cbr{a, b}$ and the matrices $P$ in $\bp{T}^{[\pi, \pi']}$
  are constructed with $p^{[\pi, \pi'] a}_{s,s'} = p^{\pi(s)}_{s, s'}$ and
  $p^{[\pi, \pi'] b}_{s,s'} = p^{\pi'(s)}_{s, s'}$. The reward function is
  defined analogously by $\bp{R}^{[\pi, \pi']} = \del{ (\vec{r}^{[\pi, \pi']
    a}_{\downarrow}, \vec{r}^{[\pi, \pi'] a}_{\uparrow}), (\vec{r}^{[\pi, \pi']
  b}_{\downarrow}, \vec{r}^{[\pi, \pi'] b}_{\uparrow}) }$ with 
  \begin{align*} 
    \vec{r}^{[\pi, \pi'] a}_{\downarrow s} & = \vec{r}^{\pi(s)}_{\downarrow s},
    \vec{r}^{[\pi, \pi'] a}_{\uparrow s} = \vec{r}^{\pi(s)}_{\uparrow s} \\
    \vec{r}^{[\pi, \pi'] b}_{\downarrow s} & = \vec{r}^{\pi'(s)}_{\downarrow s},
    \vec{r}^{[\pi, \pi'] b}_{\uparrow s} = \vec{r}^{\pi'(s)}_{\uparrow s}.
  \end{align*}

  It is easy to see that the policies $\pi$ and $\pi'$ can be executed in the
  new SBMDP $\mathcal{P}^{[\pi, \pi']}$. As all action changes from $\pi$ lead
  to smaller value vectors in each component, we can see that $\pi$ is locally
  optimal for each component, and thus, $\pi$ is optimal for all components.
  Hence, $\pi$ is an optimal policy in $\mathcal{P}^{[\pi, \pi']}$. Furthermore,
  $\pi'$ is then dominated by $\pi$ in all states and all components in
  $\mathcal{P}^{[\pi, \pi']}$ as well as in $\mathcal{P}$. Consequently, $\pi'$
  cannot lie on the Pareto frontier, which contradicts the initial assumption. 

  As we have arrived at a contradiction, we conclude that there must exist a
  state $s$ where it is $\vec{v}^{(\pi)} \not> \vec{v}^{(\pi^{(s, \pi'(s))})}$, and,
  since $d(\pi^{(s, \pi'(s))}, \pi') = c - 1$ and $d(\cdot, \cdot)$ can never
  exceed $|S|$, there exists, by induction hypothesis, a sequence of policies
  $\pi^{(s, \pi'(s))} = \pi_1, \pi_2, \ldots, \pi_c = \pi'$ for which
  $\vec{v}^{(\pi_i)} \not> \vec{v}^{(\pi_{i+1})}$. As $d(\pi, \pi^{(s,
  \pi'(s))}) = 1$, this concludes the proof.
\end{proof}
\CorrectnessExactAlgorithm*
\begin{proof}
	The correctness of the algorithm follows from
	Lemma~\ref{lma:non-dominated-path}. In detail,
	Algorithm~\ref{alg:ppareto-bfs} stores a set $P$ of policies. In the
	$i$-th step, the set $P$ is updated with policies that have distance $1$
	from already computed policies in $P$ and distance $i$ from $\pi_0$; a
	further constraint restricts the policies to be non-dominated by their
	``parent'' in $P$. This way, after $i$ steps $P$ contains all policies
	with distance $i$ from $\pi_0$ that follow a non-dominated path. By
	computing the non-dominated subset of currently found policies in line~\ref{alg:ppareto-bfs-updatefrontier}, 
	we maintain a set of mutually non-dominated policies that are reachable on a non-dominated path from
	$\pi_0$. By Lemma~\ref{lma:non-dominated-path}, this captures all
	policies from $\Ppareto$.
\end{proof}
\ifthenelse{\boolean{long}}{
%\ComplexityHeuristicAlg*
\begin{proof}
  %We would like to briefly analyze the complexity of 
  %Alg.~\ref{alg:ppareto-simple}. 
  \vh{ToDo: Check the proof.}
  The complexity of Alg.~\ref{alg:ppareto-simple} is analyzed as follows.  In
  lines~\ref{alg:ppareto-simple-iteration-start}--\ref{%
  alg:ppareto-simple-iteration-end}, $|P| = \bigO{|F||S||A|}$ policies are
  generated. For desired precision of $\epsilon$ and $B$ bits of representation
  of a floating-point value, $\bigO{\frac{B + \log \frac{1}{\epsilon} + \log
  \frac{1}{1-\gamma} + 1}{1 - \gamma}} = \bigO{\frac{B + \log \frac{1}{\epsilon}
  + \log \frac{1}{\gamma}}{1 - \gamma}}$ value update iterations are
  needed~\cite{DBLP:conf/uai/LittmanDK95}. A value update iteration can be
  performed in $\bigO{|S|^2 + |S| \log |S|}$ time steps: We need to compute the
  permutation of the value vector in $|S| \log |S|$ time steps, and then, with
  this knowledge, we have to update each state with the scalar product of the
  transition probability vector and the value vector. Together, this yields a
  total complexity of policy evaluation of all generated policies of
  $\bigO{|F| |S|^3 |A| \frac{B + \log \nicefrac{1}{\epsilon} + \log
  \frac{1}{1-\gamma}}{1-\gamma}}$.

  The complexity of computing the resulting non-dominated set that will replace
  $F$ is then $\bigO{{(|F||S||A|)}^2}$~\cite{DBLP:journals/tec/ZhangTCJ15};
  thus, the total runtime can be bounded by $\mathcal{O}(R ({(I|S||A|)}^2 +
  \frac{I|S|^3|A| (B + \log \nicefrac{1}{\epsilon} + \log \frac{1}{1 -
  \gamma})}{1 - \gamma}))$, if $R$ is the result set size and $I$ is the size of
  the largest intermediate set of non-dominated solutions. This implies cubic
  complexity in the size of the largest intermediate set in the worst case.
  Heuristically, it seems also intuitive to assume that $I$ cannot be much
  larger than the resulting set, i.\,e., $I$ and $R$ are coupled by the relation
  $I \le cR$ where $c$ is a small constant.  Furthermore, we assume that
  $\epsilon$ and $B$ are fixed. These assumptions reduce the runtime to $\bigO{
  R^3 |S|^2 |A|^2 + \frac{R^2|S|^3|A| \log \frac{1}{1 - \gamma}}{1 - \gamma} }$,
  which is roughly cubic in the resulting set size and quadratic in the size of
  state and action spaces for constant discount factors as it is usually $|S|,
  |A| = o(R)$.
\end{proof}
\section{Theoretical Complexity of Multi-Objective Optimization for SBMDPs}
\label{app:complexity}

In this section, we establish some complexity results about multi-objective optimization for SBMDPs.
These allow us, first, to constrain the algorithmic possibilities,
and second, to clearly define optimization goals.

%\subsection{Solutions with guaranteed performance}
%\label{sec:hardness}

We consider the \emph{canonical decision} problem. Informally, it amounts to
finding a policy that meets given minimum performance bounds.

\begin{definition}[\textbf{Canonical decision problem}]
	Given a vector $\vec{u} \in \R^c$ and a multi-objective optimization
	problem in a general form
	\[ 
	\max_{\vec x \in X} f_1(\vec{x}), f_2(\vec{x}), \ldots, f_c(\vec{x}),
	\]
	the \emph{canonical decision problem} is defined to decide, if a vector
	$\vec{x}^*$ such that $\vec{x}^* \in X$ and $f_j(\vec{x}^*) \ge u_j$ for
	all $j \in [c]$ exists.
\end{definition}

In the context of SBMDPs solving the canonical decision problem would 
mean finding, given vectors $\vec{u}_\downarrow, \overline{\vec{u}}, 
\vec{u}_\uparrow$, a policy $\pi$ such that the worst case, expected, 
and best case value vectors $\vec{v}_\downarrow^{(\pi)}, 
\overline{\vec{v}}^{(\pi)}, \vec{u}_\uparrow^{(\pi)}$ under $\pi$ are at 
least as large as the respective bounds $\vec{u}_\downarrow, 
\overline{\vec{u}}, \vec{u}_\uparrow$. It has been shown by 
\cite{DBLP:conf/stacs/ChatterjeeMH06} that the corresponding decision
problem is \NP-hard for general multi-objective MDPs where different
expected values are computed. The following theorem shows that the
hardness result also holds for SBMDPs even if rewards do not depend on
the chosen action and even if only two objectives are constrained by the
decision problem.

\begin{theorem}
  \label{thm:pareto-nphard}
  Deciding existence of a pure policy $\pi$ for a given stochastic
  bounded-parameter Markov decision process that delivers a worst-case
  expected discounted reward $\vec v_{\downarrow}^{(\pi)}$ and an
  average-case expected discounted reward $\overline{\vec v}^{(\pi)}$
  such that $\vec{v}_{\downarrow}^{(\pi)} \ge \vec{v}_1,
  \overline{\vec{v}}^{(\pi)} \ge \vec v_2$ for given vectors $\vec{v}_1,
  \vec{v}_2$ is \NP-complete.
\end{theorem}
\begin{proof}
  \label{sec:proof-hardness}
  The stated problem is obviously in \NP{} because the evaluation of a
  policy for a SBMDP can be done in polynomial time. As for \NP-hardness,
  we show a reduction from the subset sum problem. Given a subset sum
  instance $M = \cbr{m_1, \ldots, m_n}$ we construct a SBMDP and two
  vectors $\vec v_1, \vec v_2$ such that there is a policy $\pi$ with
  $\vec v_{\downarrow}^{(\pi)} \ge \vec v_1, \overline{\vec v}^{(\pi)} \ge
  \vec v_2$ if and only if there is a subset $I \subseteq [n]$ such that
  $\sum_{i \in I} m_i = \frac{1}{2} \sum_{i = 1}^n m_i$.

  The SBMDP which we shall construct will have an arbitrary nonzero discount
  factor $\gamma \in \intoo{0,1}$, $4n + 1$ states which have the identifiers
  $t$ and $q_i, \overline{s}_i, s_{\downarrow i}, s_{\uparrow i}$ for $i \in
  [n]$.
  The general idea of the reduction is the following: The SBMDP shall proceed
  through all states $q_i$ and end up in the absorbing state $t$. In the state
  $q_i$, it shall be possible to choose between the (adjusted for the discount
  factor) reward pairs $\del{0, 2m_i}$ and $\del{m_i, m_i}$ with the first
  component being the worst case and the second component being the average case
  reward. This way, a pure policy $\pi$ will induce a subset $I \subseteq [n]$
  such that the expected total rewards, if one starts in state $q_1$, will be
  $v_{\downarrow q_1}^{(\pi)} = \sum_{i \not\in I} m_i$ and
  $\overline{v}_{q_1}^{(\pi)} = \sum_{i=1}^n m_i + \sum_{i \in I} m_i$.
  Technically, we model this by enabling two actions in states $q_i$, $a$ and
  $b$. These actions do not generate rewards, but lead to different outcomes:
  The action $a$ leads to either $s_{\downarrow i}$ or $s_{\uparrow i}$, with
  probability in the interval $\intcc{0,1}$ and the expected value for this
  probability being $\nicefrac{1}{2}$; all actions from these two states lead,
  in turn, to $q_{i+1}$ (or $t$, if $i = n$), and the difference between
  $s_{\downarrow i}$ and $s_{\uparrow i}$ lies in the reward: the reward in
  $s_{\downarrow i}$ is $0$, the reward in $s_{\uparrow i}$ is
  $\frac{2m_i}{\gamma^{2i-1}}$.
  The action $b$ leads to the state $\overline s_i$ unconditionally with reward
  $0$, and all actions from $\overline s_i$ lead to the state $q_{i+1}$, if $i <
  n$, or $t$, if $i = n$, with reward $\frac{m_i}{\gamma^{2i-1}}$. 

  Finally, we define the expected discounted rewards we would like to get with a
  pure policy $\pi$, it should be $\frac12 \sum_{i \in [n]} m_i$ in the worst
  case and $\frac32 \sum_{i \in [n] } m_i$ in the average case.

  \begin{figure}[htp]
    \centering
    \begin{tikzpicture}[auto,>=stealth',node distance=2cm]
      \node[circle,draw] (qi) at (0,0) {$q_i$};
      \coordinate (qui) at (0,1.1);
      \node[circle,draw] (sui) at (1.1,2.2) {$s_{\uparrow i}$};
      \node[circle,draw] (sli) at (-1.1,2.2) {$s_{\downarrow i}$};
      \coordinate (qli) at (0,-1.1);
      \node[circle,draw] (sai) at (1.1, -2.2) {$\overline s_i$};
      \node[circle,draw] (qi1) at (3.3,0) {$q_{i+1}$};
      
      \path[-] (qi) edge [] node {$a, 0$} (qui);
      \path[-] (qi) edge [] node {$b, 0$} (qli);
        
      \path[->] (qli) edge [bend right] node {$p = 1$} (sai);
      \path[->] (sai) edge [bend right] node [swap] {$\cdot,
                \frac{m_i}{\gamma^{2i-1}}, p = 1$} (qi1);
        
      \path[->] (qui) edge [bend left] node [swap] {$p \in
                \intcc{0,1}, \overline{p} = \nicefrac12$}
      (sui);
      \path[->] (qui) edge [bend right] node {$p \in \intcc{0,1}, \overline{p} =
      \nicefrac12$}
      (sli);
      \path[->] (sui) edge [bend left=60] node {$\cdot,
              \frac{2m_i}{\gamma^{2i-1}}, p = 1$} (qi1);
      \path[->] (sli) edge [bend left=60] node {$\cdot, 0, p = 1$}
      (qi1);
    \end{tikzpicture}
    \caption{Construction from Theorem~\ref{thm:pareto-nphard}}
  \end{figure}

  It is easy to see that the construction can be done in polynomial time.  We
  want now to show correctness. For every subset $I \subseteq [n]$ we define a
  policy $\pi_I$ with $\pi_I(q_i) = a$ if $i \in I$ and $\pi_I(q_i) = b$ if $i
  \not \in I$. The expected discounted total reward for policy $\pi_I$ in state
  $q_1$ will then be $\sum_{i \in I} \gamma^{2i-1} \cdot
  \frac{m_i}{\gamma^{2i-1}} + \sum_{i \in I} \gamma^{2i-1} \cdot
  \frac{m_i}{\gamma^{2i-1}} = \sum_{i \in [n]} m_i + \sum_{i \in I} m_i$ in the
  average case and $\sum_{i \in [n] \setminus I} \gamma^{2i-1} \cdot
  \frac{m_i}{\gamma^{2i-1}} = \sum_{i \in [n] \setminus I} m_i$ in the worst
  case. If there is a subset $I \subseteq [n]$ such that $\sum_{i \in I} m_i =
  \sum_{i \in [n] \setminus I} m_i$, then there is a policy $\pi_I$ that yields
  the requested expected discounted rewards. Furthermore, as any pure policy
  $\pi$ induces a subset $I \subseteq [n]$ by defining $i \in I \Leftrightarrow
  \pi(q_i) = a$, the existence of a pure policy $\pi$ with the requested
  expected discounted rewards implies the existence of a subset with sum
  $\frac12 \sum_{i \in [n]} m_i$.
%
  %It is easy to see that the proof technique can also be applied to the
  %combination of best-case and expected-case measures, and best-case and
  %worst-case measures.
\end{proof}
}
\ifthenelse{\boolean{nonstationary}}{%
	\section{Proof of Theorem~\ref{thm:pareto-error-bounds}}
	\label{sec:proof-error-bounds}
	
	Since the policies computed in Algorithm~\ref{alg:ivi-pareto-popt} are
	deterministic, they can be represented by vectors $\pi_i$ for the
	decisions in the $i$-th ($1 \le i \le k$) step. Each
	$(\mat{v}_{\downarrow}, \overline{\mat{v}}, \mat{v}_{\uparrow}) \in V^k$
	results from a specific sequence $\pi_1,\ldots,\pi_k$.
	$(\mat{v}_{\downarrow}, \overline{\mat{v}}, \mat{v}_{\uparrow})$ has
	been computed starting with value vector $(\mat{0},\mat{0},\mat{0})$
	$k$ steps apart. Now assume that we start with $(\mat{w}_{\downarrow},
	\overline{\mat{w}}, \mat{w}_{\uparrow})$ instead of
	$(\mat{0},\mat{0},\mat{0})$ and use the same policy
	$\pi_1,\ldots,\pi_k$, then the resulting value vectors are
	\[\]
	\begin{array} {lll}
		\mat{v}_{\downarrow} + \gamma^{k} \prod\limits_{i=1}^k  
		P^{(\pi_i)}\mat{w}_{\downarrow} & \le \mat{v}_{\downarrow} +
		\gamma^{k} \max_{s \in S} \left(\mat{w}_{s \downarrow}\right)\idvt ,\\
		\overline{\mat{v}} + \gamma^{k} \prod\limits_{i=1}^k P^{(\pi_i)}
		\overline{\mat{w}} & \le \overline{\mat{v}} + \gamma^{k} \max_{s \in
			S}
		\left(\overline{\mat{w}}_s\right)\idvt & \mbox{ and }\\
		\mat{v}_{\uparrow} + \gamma^{k} \prod\limits_{i=1}^k  
		P^{(\pi_i)}\mat{w}_{\uparrow} & \le \mat{v}_{\uparrow} +
		\gamma^{k} \max_{s \in S} \left(\mat{w}_{s \uparrow}\right)\idvt.
	\end{array}
	\]
	Consequently, $(\mat{v}_{\uparrow}, \overline{\mat{v}},
	\mat{v}_{\downarrow})$ contains the discounted reward accumulated in the
	last $k$ steps and is a lower bound for the discounted reward
	accumulated over an infinite number of steps. Since $V^k$ contains all
	Pareto optimal vectors reachable after $k$ steps, $v^*_{\downarrow}$ is
	the maximal lower bound of the accumulated reward in $k$ steps,
	$\overline{v}^*$ is the maximal average reward that can be accumulated
	in $k$ steps and $v^*_{\uparrow}$ is the maximal upper bound of the
	accumulated reward in $k$ steps. This implies
	\[\]
	\begin{array}{llll}
		\max_{s \in S} \left(\mat{w}_{s \downarrow} \right) & \le
		\sum\limits_{i=0}^\infty \left(\gamma^{k} \right)^i v_{\downarrow}^* &
		= \frac{1}{1-\gamma^{k}}v_{\downarrow}^* , \\
		\max_{s \in S} \left(\overline{\mat{w}}_s \right) & \le
		\sum\limits_{i=0}^\infty \left(\gamma^{k} \right)^i \overline{v}^* & =
		\frac{1}{1-\gamma^{k}}\overline{v}^* & \mbox{ and }\\
		\max_{s \in S} \left(\mat{w}_{s \uparrow} \right) & \le
		\sum\limits_{i=0}^\infty \left(\gamma^{k} \right)^i v_{\uparrow}^* &
		= \frac{1}{1-\gamma^{k}}v_{\uparrow}^*
	\end{array}
	\]
	for every value vector $(\mat{w}_{\downarrow}, \overline{\mat{w}},
	\mat{w}_{\uparrow})$ reachable by an arbitrary policy and also for every
	value vector $(\mat{w}_{\downarrow}, \overline{\mat{w}},
	\mat{w}_{\uparrow}) \in \mathcal{V}_{opt}$. 
	
	Let $(\mat{u}_{\downarrow}, \overline{\mat{u}}, \mat{u}_{\uparrow}) \in
	\mathcal{V}_{opt}$ be a Pareto optimal pair of vectors resulting from
	policy $f = \pi_1,\ldots, \pi_k, \ldots \in \mathcal{P}_{opt}$. First,
	assume that $\pi_1,\ldots,\pi_k$ is a policy that has been used in
	Algorithm~\ref{alg:ivi-pareto-popt} to compute $(\mat{v}_{\downarrow},
	\overline{\mat{v}}, \mat{v}_{\uparrow}) \in V^k$. Then
	$\overline{\mat{v}}$ contains the average reward accumulated during the
	last $k$ steps by policy $f$. Additionally, $\overline{\mat{u}}$
	contains the discounted reward accumulated in the steps more than $k$
	steps apart. A bound for this reward is given above and this bound is
	weighted by $\gamma^{k}$ which results in the bound given in the
	theorem. With the same arguments the bound for the worst and best case
	behavior can be derived.
	
	Now assume that $f = \pi_1,\ldots, \pi_k; \ldots \in \mathcal{P}_{opt}$
	but $\pi_1,\ldots,\pi_k$ has not been used to compute a vector from
	$V^k$. Then $(\mat{u}_{\downarrow}, \overline{\mat{u}},
	\mat{u}_{\uparrow})$ consists of a part $(\mat{w}_{\downarrow},
	\overline{\mat{w}}, \mat{w}_{\uparrow})$ that includes the reward
	accumulated during the last $k$ steps and the rest which can be bounded
	as in the previous case. However, since $V^k$ contains all Pareto
	optimal vectors with respect to the reward accumulated in $k$ steps, it
	contains a vector pair $(\mat{v}_{\downarrow}, \overline{\mat{v}},
	\mat{v}_{\uparrow})$ with $\mat{v}_{\downarrow} \ge
	\mat{w}_{\downarrow}$, $\overline{\mat{v}} \ge \overline{\mat{w}}$ and
	$\mat{v}_{\uparrow} \ge \mat{w}_{\uparrow}$. These vectors plus the
	bound for the reward accumulated in the steps $k+1, \ldots$ define a
	bound for $(\mat{u}_{\downarrow}, \overline{\mat{u}},
	\mat{u}_{\uparrow})$. 
}{}

%We would like to note that this result also applies to other
%multi-objective measures, such as the combination of best-case and
%expected-case measures, and best-case and worst-case measures.
%\end{comment}

\ifthenelse{\boolean{complexity}}{%
	\subsection{Pareto optimality measures}
	\label{sec:pareto-hardness}
	
	By generalizing the considered problem as a variant of a multi-objective
	Markov decision process problem, with the components being the
	worst-case, best-case, or average-case expected discounted rewards, we
	can consider general measures for optimality in multi-objective
	optimization problems. Hence, it is tempting to consider known
	techniques for multi-objective optimization in this context.
}{}

% For a given weight vector, Algorithm~\ref{alg:scalarization} describes 
% a variant of value iteration which converges to an optimal pure policy. 
% We denote the corresponding policy as $\pi^{\vec w}$. However, if $\vec 
% w$ is not known a priori, one is usually interested in the complete set 
% of pure policies which are optimal for some weight vector. More 
% formally, let $\mathcal{W} = \{ \vec w = (w_\downarrow, \overline{w}, 
% w_\uparrow) \mid \vec w \ge \vec 0, \vec w \idvt = 1\}$ the set of 
% weight vectors and let $\Ppure$ be the set of pure policies. Then define
% \[
% \Pweight = \left\{\pi \mid \pi \in \Ppure,
% \exists \vec
% w \in \mathcal{W}: \pi = \pi^{\vec w}\right\}
% \]
% as the set of pure policies that are optimal for some weight vector.

% Furthermore, it is easy to see that the vectors in $\Vweight$ span the 
% convex hull of $\Vpure$. Nevertheless, the number of policies in 
% $\Pweight$ can still be exponential in the size of the MDP.

\ifthenelse{\boolean{nonstationary}}{%
	Up to now we have only considered optimal stationary policies and in
	our setting pure policies are of special importance. However, if we
	allow general policies which consider the history, then additional
	Pareto optimal solutions may be found. This will be shown by a simple
	example (shown in Fig.~\ref{ex:simple}). The SBMDP has two states,
	in state $1$ one can choose between action $a$ and $b$. If $a$ is
	selected the process goes with probability $p$, which is uniformly $[0,1]$
	distributed, into state $2$ and stays with probability $1-p$ in state
	$1$. If $b$ is selected, then the process makes with probability $q$,
	which is uniformly $[0.5,0.7]$ distributed, a transition to state $2$
	and stays with probability $1-q$ in state $1$. From state $2$ the
	process always returns to state $1$ with probability $1$ no matter
	which action has been chosen.  The reward vector is $\vec{r}^a =
	\vec{r}^b = {(1, 0)}^\top$. 
	
	\begin{figure}[htp]
		\centering
		\begin{tikzpicture}[auto,>=stealth',node distance=2cm]
		\node[circle,draw] (1) at (0,3) {$1$};
		\node[circle,draw] (2) at (2,0) {$2$};
		
		\path[->] (1) edge [bend left=60] node {a$[0,1]$, b$[0.5,0.7]$} (2);      
		\path[->] (2) edge [bend right=-60] node {*$[1,1]$} (1);
		\path[->] (1) edge [loop left=90] node {a$[0,1]$, b$[0.3,0.5]$} (1);
		
		\end{tikzpicture}
		\caption{Simple example SBMDP}
		\label{ex:simple}
	\end{figure}
	
	Since decisions in state $2$ are irrelevant, we have two pure policies
	namely $(a)$ which chooses always action $a$ and $(b)$ that chooses
	always $b$. Obviously, $(a)$ is better for maximizing the average 
	reward and $b$ is better for the worst case reward. For $\gamma=0.9$ the corresponding
	values are $\vec{v}_1^{(a)} = \left(v_{1
		\downarrow}^{(a)}, \overline{v}_1^{(a)}, v_{1\uparrow}^{(a)}\right) = \left(5.2632,
	6.8966, 10.000\right)$ and  
	$\vec{v}_1^{(b)} = \left(v_{1 \downarrow}^{(b)}, \overline{v}_1^{(b)},
	v_{1 \uparrow}^{(b)}\right) =
	\left(6.1350, 6.4935, 6.8966\right)$. Both vectors
	are Pareto optimal. Now consider the non-stationary polices $ab$ and
	$ba$ which collect alternating sequences of the actions $a$ and
	$b$ in state $1$. We have $\vec{v}_1^{(ab)} = \left(v_{1 \downarrow}^{(ab)},
	\overline{v}_1^{(ab)},v_{1 \uparrow}^{(ab)}\right) = \left(5.2632,
	6.7024, 7.8684\right)$ and 
	$\vec{v}_1^{(ba)} = \left(v_{1 \downarrow}^{(ba)},
	\overline{v}_1^{(ba)}, v_{1 \uparrow}^{(ba)}\right) = \left(5.5913,
	6.6398, 7.6316\right)$. $\vec v_1^{(ab)}$ is dominated by
	$\vec{v}_1^{(a)}$ whereas $\vec v_1^{(ba)}$ is another Pareto optimal
	vector which cannot be reached with any stationary policy. By
	extending the history, we obtain additional Pareto optimal
	policies. Figure~\ref{fig0} shows the Pareto optimal value vectors for state $1$
	according the minimum and maximum which have been computed with the
	algorithm presented below. For an relative error of $0.001$, $155$
	Pareto optimal value vectors are computed for the minimum and maximum
	and $428$ Pareto optimal vectors considering all three goals, minimum,
	maximum and average value. 
	
	\begin{figure}
		\centering
		\includegraphics[width=118mm]{First.pdf} 
		\caption{Pareto optimal values for the gain of state $1$ in the
			simple example.}
		\label{fig0}
	\end{figure}
	
	We define $\Popt$ as the set of Pareto optimal policies. In general, 
	$\Popt$ can be an infinite set but since decisions $k$ steps apart 
	receive a discounting factor $\gamma^k$, the influence of past decisions 
	vanishes exponentially. Again we denote by $\Vopt$ the set of value 
	vectors belonging to policies from $\Popt$.
}{}

}
\end{document}